\documentclass[twoside,11pt]{article}

\usepackage{amsmath,amssymb,amsfonts}
\usepackage{amsthm}

\makeatletter
\let\PaperOrigNewtheorem\newtheorem
\let\PaperOrigProof\proof
\let\PaperOrigEndProof\endproof
\def\PaperSkipThm#1{\@ifnextchar[{\PaperSkipThmOpt}{\PaperSkipThmNoOpt}}
\def\PaperSkipThmOpt[#1]#2{}
\def\PaperSkipThmNoOpt#1{}
\let\newtheorem\PaperSkipThm
\makeatother
\usepackage[preprint]{jmlr2e}
\let\newtheorem\PaperOrigNewtheorem
\let\proof\PaperOrigProof
\let\endproof\PaperOrigEndProof

\usepackage{multirow}
\usepackage{mathrsfs}
\usepackage{xcolor}
\usepackage{textcomp}
\usepackage{booktabs}
\usepackage{placeins}

\theoremstyle{plain}
\newtheorem{theorem}{Theorem}
\newtheorem{proposition}[theorem]{Proposition}
\newtheorem{lemma}[theorem]{Lemma}
\newtheorem{corollary}[theorem]{Corollary}
\newtheorem{assumption}{Assumption}
\theoremstyle{remark}

\newtheorem{remark}{Remark}
\newtheorem{hypothesis}{Hypothesis}

\theoremstyle{definition}

\ShortHeadings{Training-Time BN Reshapes Local Partition Geometry}{Qi, Wei, Yu, Shen, Murino and Beyan}
\firstpageno{1}

\begin{document}

\title{Training-Time Batch Normalization Reshapes Local Partition Geometry in Piecewise-Affine Networks}

\author{\name Xuan Qi\thanks{Equal contribution.}\ \email xuan.qi@iit.it \\
       \addr AI for Good (AIGO), Istituto Italiano di Tecnologia, Genoa, Italy\\
       DITEN, University of Genoa, Genoa, Italy
       \AND
       \name Yi Wei\footnotemark[1] \email ywei@smail.nju.edu.cn \\
       \addr State Key Laboratory of Novel Software Technology, Nanjing University, Nanjing, China\\
       School of Intelligence Science and Technology, Nanjing University, Suzhou, China
       \AND
       \name Fanqi Yu  \email fanqi.yu@iit.it \\
       \addr AI for Good (AIGO), Istituto Italiano di Tecnologia, Genoa, Italy\\
       DITEN, University of Genoa, Genoa, Italy
       \AND
       \name Furao Shen \thanks{Corresponding author.}\ \email frshen@nju.edu.cn \\
       \addr State Key Laboratory of Novel Software Technology, Nanjing University, Nanjing, China\\
       School of Artificial Intelligence, Nanjing University, Nanjing, China
       \AND
       \name Vittorio Murino \email vittorio.murino@iit.it \\
       \addr AI for Good (AIGO), Istituto Italiano di Tecnologia, Genoa, Italy\\
             DIBRIS, University of Genoa, Genoa, Italy
       \AND
       \name Cigdem Beyan \email cigdem.beyan@univr.it \\
       \addr AI for Good (AIGO), Istituto Italiano di Tecnologia, Genoa, Italy\\
       Department of Computer Science, University of Verona, Verona, Italy}

\maketitle
\thispagestyle{plain}

\begin{abstract}
The effect of batch normalization (BN) on the function realized during training is less understood than its effect on optimization. We study training-time BN in continuous piecewise-affine (CPA) networks through the affine regions they induce in local $\ell_\infty$ windows around the data. Conditioned on a mini-batch, every switching hyperplane of a BN layer with learned parameters and a general CPA activation is a translate of a hyperplane through the batch centroid, at an offset given by a standardized quantity times the batch spread, independent of the raw bias. We prove an exact formula, valid under general position, for the number of regions in a window in terms of window cuts and interior intersections. For a BN layer it gives a lower bound on the local count in terms of the standardized offsets; the bound depends on the batch only through its spread and not on the biases or the batch location. Without BN the best such bound is $1$: for any weights and batch, some biases leave the window as a single region. We prove that the formula holds in deeper layers, inside parent affine regions of any rank. By the exact formula, BN increases the expected local count whenever it increases the expected numbers of window cuts and of interior intersections, which we examine empirically. Offset distributions of networks trained on synthetic data, MNIST, CIFAR-10 and TinyImageNet show higher window-cut rates under BN, and exact enumeration in two dimensions attributes most of the additional BN regions to interior intersections.
\end{abstract}


\section{Introduction}\label{sec:intro}

Batch Normalization (BN)~\citep{book1} is a central component of modern deep neural networks and is used in a wide range of architectures. Its empirical benefits are well documented, and much of the existing theory explains these benefits through optimization-related effects, including improved conditioning, smoother objectives, stabilized gradients, and the ability to use larger effective learning rates~\citep{book2,book3,book4,book8,book9,book10,book11,book12,book13}. While these explanations clarify why BN facilitates optimization, they leave largely open a complementary question: \emph{how does BN affect the function that the network realizes during training?} In particular, beyond its role in optimization, it remains unclear how BN influences the structure and complexity of the learned input--output mapping.

This question is especially natural for networks built from piecewise-affine nonlinearities such as ReLU~\citep{book26}, LeakyReLU~\citep{book48}, and, more generally, continuous piecewise-affine (CPA) activations. Such networks realize CPA maps whose input--output behavior is organized by a partition of the input space into affine regions. Within each region, the network acts as an affine map, and nonlinear expressivity arises from the number, location, and arrangement of these regions relative to the data. From this perspective, the affine-region partition provides a natural geometric object for studying the realized behavior of CPA networks, and a substantial literature has used affine regions and switching sets to analyze the expressivity of ReLU and related architectures~\citep{book5,book6,book7,book14,book15,book16,book17,book18,book19,book20,book21,book22,book23,book24,book53}. Understanding the effect of BN at the function level therefore requires understanding how it changes the switching-hyperplane arrangement and the induced affine-region partition, especially in neighborhoods relevant to the data distribution.

This leads to the central question of the paper:
\begin{quote}
\emph{How does training-time BN reshape the CPA partition near the data?}
\end{quote}

Three features make this question difficult. First, BN is a \emph{training-time}, \emph{batch-dependent} operation: its centering and scaling rely on mini-batch statistics, so the effective pre-activation of each example depends on the other examples in its mini-batch. Second, in CPA networks, even small changes in pre-activation geometry can shift switching hyperplanes and alter local region boundaries. Third, although the inference-time form of BN can be absorbed into an affine reparameterization, this does not capture the batch-dependent geometry that arises during training. A function-level effect of BN must therefore be studied through the training-time geometry of the partition rather than through the inference-time function alone.

In this paper we take this geometric viewpoint. We study training-time BN in CPA networks through the breakpoint-switching hyperplanes that define the affine-region partition. Since our focus is on local behavior near the data rather than worst-case global complexity, we introduce a local quantity, \emph{local region density}, defined by exact affine-region counts inside $\ell_\infty$ neighborhoods. This gives a quantity that can be computed in the same way for BN and non-BN models.

Conditioned on a mini-batch, BN places the switching hyperplanes of a layer relative to the batch itself: each switching hyperplane is a translate of a hyperplane through the batch centroid, at an offset measured in units of the batch spread. Two questions then organize the theory. First, when does a switching hyperplane cut a local neighborhood of the data? Second, when do such cuts produce a finer local partition?

The contributions are as follows.
\begin{itemize}
    \item \textbf{Exact local geometry of training-time BN.}
    Conditioned on a mini-batch, every breakpoint-switching hyperplane of a BN layer is a translate of a hyperplane through the batch centroid, at an offset equal to the standardized quantity $(\tau_q-\beta_j)/\gamma_j$ times the spread of the batch along the weight vector, and independent of the raw bias. This extends the central-arrangement observation of \citet{book30}, made for ReLU with $\beta=0$, to learned shifts and CPA activations with several breakpoints. We also prove an exact criterion for when a switching hyperplane cuts an $\ell_\infty$ window and, under general position, an exact formula for the number of affine regions in the window in terms of window cuts and interior intersections.

    \item \textbf{A lower bound that holds uniformly over the biases only under BN.}
    Under general position, the number of affine regions of a BN layer in an $\ell_\infty$ window of radius $r$ centered at the batch centroid is at least one plus the number of window cuts and interior intersections certified by an explicit condition on the standardized offsets, the hyperplane directions and the batch spread. For a single hyperplane the condition is that the absolute standardized offset times the batch spread, up to the $\varepsilon$ term of BN, be smaller than $r$. The bound depends on the batch only through its spread and holds uniformly over the biases and the location of the batch. For a layer without BN the best such bound is the trivial bound $1$: for any weights and any batch, some choice of biases leaves the window as a single region.

    \item \textbf{Exact counts in deeper layers.}
    For a hidden layer beyond the first, consider a parent region, that is, a region of the input space on which the preceding layers act as a single affine map. Take a window in representation space and its preimage in the parent region. We prove that the switching hyperplanes of the layer divide this preimage into as many regions as they divide the part of the window that lies in the image of the parent region, whatever the rank of the affine map. The exact formula therefore also gives local region counts of deeper layers in the input space.

    \item \textbf{The BN/non-BN comparison and its empirical test.}
    By the exact formula, BN increases the expected local region count whenever it increases the expected numbers of window cuts and of interior intersections. We state these two orderings as hypotheses and test them. Offset distributions of networks trained on synthetic data, MNIST, CIFAR-10 and TinyImageNet show higher window-cut rates under BN, and exact region enumeration in two dimensions shows that most of the difference in local region counts comes from intersections inside the window.
\end{itemize}

These results complement optimization-based accounts of BN. They describe how BN, by placing switching hyperplanes relative to the batch, affects the CPA partition near the data during training. This view does not replace the standard explanations of BN; it adds a function-level description that is natural for piecewise-affine networks.

The remainder of the paper is organized as follows. Section~\ref{sec:relatedwork} reviews related work. Section~\ref{sec:methodology} introduces the CPA and BN framework and defines local region density. Section~\ref{sec:bn-local-density-standard-linf} develops the theory: the switching geometry of training-time BN and exact local counts, their control by standardized offsets, the reduction of the BN/non-BN comparison to two hypotheses, and the transfer through depth. Section~\ref{sec:experiments} presents the experiments, and Sections~\ref{sec:discussion}--\ref{sec:conclusion} discuss the results, their limitations and conclusions.

\section{Related Work}
\label{sec:relatedwork}

Our work lies at the intersection of two research directions: analyses of BN and geometric studies of CPA networks. Relative to these literatures, our focus is on the \emph{training-time geometry of the realized function}. In particular, we study how batch-dependent normalization reshapes the local arrangement of switching hyperplanes near the data, rather than focusing only on optimization effects or on worst-case static properties of CPA architectures.

\subsection{Batch Normalization Beyond Optimization}

Since its introduction, BN has largely been studied through optimization-related mechanisms. Early work attributed its effectiveness to reducing internal covariate shift~\citep{book1}, while later analyses emphasized reparameterization effects, smoother objectives, more stable gradient dynamics, and improved conditioning~\citep{book2,book3,book4,book27,book28}. These explanations are important and largely complementary, but they primarily address optimization behavior rather than the geometry of the function realized during training.

A smaller body of work studies how BN affects the partition of CPA networks. \citet{book29} examine empirically how training choices such as BN and dropout affect the properties of linear regions. \citet{book30} analyze BN with $\gamma=\mathbf 1$ and $\beta=0$ through the spline partition of CPA networks. They prove that the BN mean minimizes the total least-squares distance between a mini-batch and the hyperplanes of a layer, so that the arrangement of each layer passes through the batch centroid, and they observe at random initialization that partition boundaries concentrate near the data. Our Corollary~\ref{cor:through-centroid} recovers their central arrangement as the special case of ReLU with $\beta=0$. With learned parameters $(\gamma,\beta)$ and activations with several breakpoints, the arrangement is in general not central, and we describe it through the standardized offsets. We study window cuts and exact local region counts, transfer these counts through depth, and measure them during and after training.

\subsection{Geometry of Continuous Piecewise-Affine Networks}

A substantial literature studies the expressivity of ReLU and more general CPA networks through the affine regions induced by activation patterns~\citep{book50,book51}. Foundational results showed that the number of linear regions can grow exponentially with depth and only polynomially with width, thereby identifying depth as a central source of geometric complexity~\citep{book5,book6,book7,book14,book20,book21,book22}. Subsequent work refined these insights through sharper combinatorial bounds and approximation-theoretic interpretations~\citep{book16,book17,book18,book19,book23,book33,book34,book37,book38,book42,book43,book44,book45,book46}. More algorithmic studies developed exact counting methods and visualization tools for CPA partitions and decision boundaries~\citep{book15,book24,book35,book36,book52}. Recent work has also used the piecewise-affine structure of ReLU networks to study complementary function-level questions, including parameter identifiability and robustness verification~\citep{book54,book55}.

Our work differs from this literature in two respects. We study how the partition is refined in neighborhoods of the data, not global or worst-case region counts, and we analyze BN at training time, conditioned on a mini-batch. Several recent works examine how initialization or training shapes the partition~\citep{book39,book40,book41}, but none of them isolates the effect of BN. Our analysis combines a batch-conditional description of the switching hyperplanes with exact region enumeration.

\section{Preliminaries and Geometric Framework}
\label{sec:methodology}

This section introduces the networks and their affine regions, describes how BN is used during training and the frozen-batch map that we analyze, and defines the local region counts used throughout the paper.

\subsection{Networks, activations, and affine regions}
\label{sec:affine-regions}

\paragraph{Notation.}
Scalars and vectors are denoted by lowercase letters and matrices by uppercase letters. Boldface is used only for vectors whose entries are themselves indexed objects, such as $\mathbf M(r)=(M_1(r),\dots,M_n(r))$. For a vector $v$, $v_j$ is its $j$th entry; $\operatorname{Diag}(v)$ is the diagonal matrix with diagonal $v$, and $\odot$ is the Hadamard product. We write $\operatorname{int}(A)$ for the interior of a set $A$, $\#A$ for the number of elements of a finite set, and $\mathbf 1\{E\}$ for the indicator of an event or condition $E$.

\paragraph{Networks.}
We consider feedforward networks of depth $L\in\mathbb N$ with widths $(D_0,\dots,D_L)$, indexed by $l\in\{0,\dots,L\}$; $D_0$ is the input dimension and $D_L$ the output dimension. Weight layers are indexed by $l\in\{1,\dots,L\}$, with weights $W^{(l)}\in\mathbb R^{D_l\times D_{l-1}}$ and biases $b^{(l)}\in\mathbb R^{D_l}$, and hidden layers by $l\in\{1,\dots,L-1\}$. Representations are indexed by $l\in\{0,\dots,L-1\}$, with $h^{(0)}(x):=x\in\mathbb R^{D_0}$, and for hidden layers
\begin{equation}
\label{eq:pre-post}
z^{(l)}(x)=W^{(l)}h^{(l-1)}(x)+b^{(l)},
\qquad
h^{(l)}(x)=\sigma\!\bigl(z^{(l)}(x)\bigr),
\end{equation}
where $\sigma$ is applied elementwise. The output layer is linear, $f(x;\theta)=W^{(L)}h^{(L-1)}(x)+b^{(L)}$, and $\theta$ collects all parameters.

\paragraph{CPA activations.}
The activation $\sigma:\mathbb R\to\mathbb R$ is continuous and piecewise affine (CPA) with finitely many pieces: there are breakpoints
\[
-\infty=\tau_0<\tau_1<\cdots<\tau_{K-1}<\tau_K=+\infty
\]
and coefficients $\{(a_k,\eta_k)\}_{k=1}^K\subset\mathbb R^2$ such that
\begin{equation}
\label{eq:cpa-sigma}
\sigma(t)=a_k t+\eta_k
\qquad\text{for }t\in(\tau_{k-1},\tau_k),\quad k=1,\dots,K,
\end{equation}
with $a_k\tau_k+\eta_k=a_{k+1}\tau_k+\eta_{k+1}$ for $k=1,\dots,K-1$. ReLU is the case $K=2$, $\tau_1=0$, $(a_1,\eta_1)=(0,0)$, $(a_2,\eta_2)=(1,0)$. Compositions of affine and CPA maps are CPA, so $f(\cdot;\theta)$ is CPA on $\mathbb R^{D_0}$.

\paragraph{Switching set and affine regions.}
Let $\mathcal N:=\{(l,j):l=1,\dots,L-1,\ j=1,\dots,D_l\}$ index the hidden units, and let $\pi:\mathbb R\setminus\{\tau_1,\dots,\tau_{K-1}\}\to\{1,\dots,K\}$ be the piece selector, $\pi(t)=k$ if and only if $t\in(\tau_{k-1},\tau_k)$. The \emph{switching set} is
\begin{equation}
\label{eq:switching-set}
\Psi
:=
\bigcup_{(l,j)\in\mathcal N}\ \bigcup_{r=1}^{K-1}
\bigl\{x\in\mathbb{R}^{D_0}: z_j^{(l)}(x)=\tau_r\bigr\}.
\end{equation}
For $x\notin\Psi$ the \emph{activation pattern} $p(x)\in\{1,\dots,K\}^{|\mathcal N|}$ is $[p(x)]_{(l,j)}:=\pi(z_j^{(l)}(x))$. An \emph{affine region} is a connected component of $\mathbb R^{D_0}\setminus\Psi$, and $\mathcal R(f)$ denotes the set of affine regions. The activation pattern is constant on each affine region. Conversely, on the set of inputs whose patterns at layers $1,\dots,l-1$ are fixed, $z^{(l)}$ is affine, so the set of inputs with a given full pattern is an intersection of open slabs and half-spaces. Hence every affine region is an open convex polyhedron.

For $\Psi$ to be negligible, the following condition is needed.

\begin{assumption}[Non-degenerate parameters]
\label{ass:nondegenerate}
For every hidden unit $(l,j)\in\mathcal N$, every breakpoint $\tau_r$, and every set $R$ with nonempty interior on which the activation pattern of layers $1,\dots,l-1$ is constant, the affine function $z^{(l)}_j|_R$ is either non-constant or its constant value differs from $\tau_r$.
\end{assumption}

Without this condition $\Psi$ can contain an open set: if all units feeding a ReLU unit are inactive on an open set, the pre-activation of that unit equals its bias there, and a zero bias then makes a whole open set lie in $\Psi$. Under Assumption~\ref{ass:nondegenerate}, every level set $\{x:z_j^{(l)}(x)=\tau_r\}$ is contained in a finite union of polyhedra of dimension at most $D_0-1$, so $\Psi$ has Lebesgue measure zero and the affine regions cover $\mathbb R^{D_0}$ up to a null set. Since the constant value of $z_j^{(l)}|_R$ is affine in $b_j^{(l)}$ with coefficient one, and there are finitely many activation patterns, the assumption fails only on a Lebesgue-null set of parameters. For BN models it is imposed on the frozen-batch map of Section~\ref{sec:bn}, where $\beta^{(l)}_j$ plays the role of $b^{(l)}_j$.

\paragraph{Affine form on a region.}
Fix $R\in\mathcal R(f)$. For each hidden layer $l=1,\dots,L-1$ and unit $j=1,\dots,D_l$, let $a_j^{(l)}(R):=a_{[p(x)]_{(l,j)}}$ and $c_j^{(l)}(R):=\eta_{[p(x)]_{(l,j)}}$ for any $x\in R$, and let $D^{(l)}(R):=\operatorname{Diag}(a^{(l)}(R))$ be the \emph{slope-gating matrix}. Then $h^{(l)}(x)=D^{(l)}(R)z^{(l)}(x)+c^{(l)}(R)$ on $R$. Writing $z^{(l)}(x)=A_R^{(l)}x+b_R^{(l)}$ on $R$, we have $A_R^{(1)}=W^{(1)}$, $b_R^{(1)}=b^{(1)}$ and, for $l=2,\dots,L-1$,
\begin{equation}
\label{eq:affine-rec-hidden-cpa}
A_R^{(l)}=W^{(l)}D^{(l-1)}(R)A_R^{(l-1)},
\qquad
b_R^{(l)}=W^{(l)}\bigl(D^{(l-1)}(R)b_R^{(l-1)}+c^{(l-1)}(R)\bigr)+b^{(l)}.
\end{equation}
The network is affine on $R$, $f(x)=A_Rx+b_R$, with $A_R$ and $b_R$ given by the same recursion at $l=L$ without the final activation.

\paragraph{Prefix maps and parent regions.}
For a hidden layer $\ell\ge2$, the \emph{prefix map} $h^{(\ell-1)}$ is the CPA map computed by layers $1,\dots,\ell-1$. Its affine regions, defined as above with the hidden units of layers $1,\dots,\ell-1$ only, are called \emph{parent regions} of layer $\ell$. On a parent region $R$,
\begin{equation}
\label{eq:parent-affine-map}
\begin{gathered}
h^{(\ell-1)}(x)=\widetilde A_Rx+\widetilde d_R\quad(x\in R),\\
\widetilde A_R:=D^{(\ell-1)}(R)A_R^{(\ell-1)},
\qquad
\widetilde d_R:=D^{(\ell-1)}(R)b_R^{(\ell-1)}+c^{(\ell-1)}(R).
\end{gathered}
\end{equation}
We use $(\widetilde A_R,\widetilde d_R)$ for the prefix map and $(A_R,b_R)$ for the full network.

\subsection{Batch normalization during training}
\label{sec:bn}

\paragraph{Mini-batch training.}
Let $\mathcal D=\{x_1,\dots,x_N\}$ be the training set and $M$ the batch size. At the start of each epoch the training set is randomly permuted and split into consecutive blocks of at most $M$ examples; each block is a mini-batch. A training step processes one mini-batch $\{x_{i_1},\dots,x_{i_M}\}$: it computes the forward pass described below, evaluates the average loss over the mini-batch, and updates all parameters by one step of a gradient-based optimizer. We write $h_k^{(l)}$ for the layer-$l$ representation of the $k$th example of the current mini-batch under the current parameters, with $h_k^{(0)}=x_{i_k}$.

\paragraph{Forward pass with BN.}
Fix a hidden layer $l$. The pre-activations of the mini-batch are $z_k^{(l)}=W^{(l)}h_k^{(l-1)}+b^{(l)}$, $k=1,\dots,M$. Training-time BN computes the per-feature batch statistics
\begin{equation}
\label{eq:bn-stats}
\mu_j^{(l)}=\frac{1}{M}\sum_{k=1}^M z_{k,j}^{(l)},
\qquad
v_j^{(l)}=\frac{1}{M}\sum_{k=1}^M \bigl(z_{k,j}^{(l)}-\mu_j^{(l)}\bigr)^2,
\qquad
j=1,\dots,D_l,
\end{equation}
and the transformation
\begin{equation}
\label{eq:bn-train}
\widehat z_{k,j}^{(l)}
=
\gamma_j^{(l)}\frac{z_{k,j}^{(l)}-\mu_j^{(l)}}{\sqrt{v_j^{(l)}+\varepsilon}}
+\beta_j^{(l)},
\qquad
\varepsilon>0,
\end{equation}
where $\gamma^{(l)},\beta^{(l)}\in\mathbb R^{D_l}$ are learnable and initialized at $\gamma^{(l)}=\mathbf 1$ and $\beta^{(l)}=0$. Throughout the paper BN is applied before the nonlinearity, so $h_k^{(l)}=\sigma(\widehat z_k^{(l)})$. Because $\mu^{(l)}$ and $v^{(l)}$ depend on the whole mini-batch, the representation of an example depends on the other examples of its mini-batch, and gradients are taken through the batch statistics.

\paragraph{Running statistics and inference.}
After every training step, running estimates are updated as exponential moving averages,
\begin{equation}
\label{eq:running-stats}
\bar\mu^{(l)}\leftarrow(1-m)\,\bar\mu^{(l)}+m\,\mu^{(l)},
\qquad
\bar v^{(l)}\leftarrow(1-m)\,\bar v^{(l)}+m\,\tfrac{M}{M-1}\,v^{(l)},
\end{equation}
with momentum $m\in(0,1)$ and initial values $\bar\mu^{(l)}=0$, $\bar v^{(l)}=\mathbf 1$, as in standard implementations. At inference time the batch statistics are replaced by these estimates:
\begin{equation}
\label{eq:bn-infer}
\widehat z^{(l)}
=
\gamma^{(l)}\odot\frac{z^{(l)}-\bar\mu^{(l)}}{\sqrt{\bar v^{(l)}+\varepsilon}}
+\beta^{(l)}
=
A^{(l)}W^{(l)}h^{(l-1)}+\widetilde b^{(l)},
\end{equation}
with $A^{(l)}:=\operatorname{Diag}\bigl(\gamma^{(l)}/\sqrt{\bar v^{(l)}+\varepsilon}\bigr)$ and $\widetilde b^{(l)}:=A^{(l)}(b^{(l)}-\bar\mu^{(l)})+\beta^{(l)}$. Thus inference-time BN is an affine reparameterization of $(W^{(l)},b^{(l)})$ and does not enlarge the class of CPA functions.

\paragraph{The frozen-batch map.}
During training the map applied to an example depends on its mini-batch. To obtain a deterministic function, we fix a \emph{reference mini-batch} $\mathcal B$ of inputs and replace, at every BN layer, the batch statistics by those of the representations of $\mathcal B$, written $\mu^{(l)}_{\mathcal B}$ and $v^{(l)}_{\mathcal B}$:
\begin{equation}
\label{eq:bn-affine-reparam-train}
\begin{gathered}
\widehat z_{\mathcal B}^{(l)}(x)
=
A_{\mathcal B}^{(l)}W^{(l)}h^{(l-1)}(x)+\widetilde b_{\mathcal B}^{(l)},\\
A_{\mathcal B}^{(l)}
:=
\operatorname{Diag}\!\Bigl(\frac{\gamma^{(l)}}{\sqrt{v_{\mathcal B}^{(l)}+\varepsilon}}\Bigr),
\qquad
\widetilde b_{\mathcal B}^{(l)}
:=
A_{\mathcal B}^{(l)}\bigl(b^{(l)}-\mu_{\mathcal B}^{(l)}\bigr)+\beta^{(l)}.
\end{gathered}
\end{equation}
The resulting map $f(\cdot;\theta\mid\mathcal B)$ is CPA, and on the examples of $\mathcal B$ it coincides with the training-time forward pass. All definitions of Section~\ref{sec:affine-regions} apply to it. By \emph{training-time BN geometry} we mean the geometry of this map.

\paragraph{Sources of randomness.}
Expectations in this paper are taken over at most three sources of randomness: (i) the initial parameters, drawn independently for each unit ($\gamma=\mathbf 1$ and $\beta=0$ for BN); (ii) the random permutation of the training set at each epoch, which determines the mini-batches and hence the stochastic gradients; and (iii) for BN models, the choice of the reference mini-batch $\mathcal B$.

\subsection{Local region counts}
\label{sec:local-density}

Let $P_X$ be a data distribution on $\mathbb R^{D_0}$ with $\mathbb E\|X\|<\infty$, and let $\bar x:=\mathbb E[X]$ be its centroid. For a measurable set $\Omega\subseteq\mathbb R^{D_0}$, the \emph{region-intersection count} is
\begin{equation}
\label{eq:region-count}
N_{\mathrm{reg}}(f,\Omega)
:=
\#\bigl\{R\in\mathcal{R}(f):R\cap\operatorname{int}(\Omega)\neq\varnothing\bigr\}.
\end{equation}
Using $\operatorname{int}(\Omega)$ makes the count insensitive to contacts with the boundary. For $x_0\in\mathbb R^{D_0}$ and $r>0$ let $B_\infty(x_0,r):=\{x:\|x-x_0\|_\infty\le r\}$. Since $f$ has finitely many affine regions, $N_{\mathrm{reg}}(f,B_\infty(x_0,r))<\infty$, and we call it the \emph{local region count}. The \emph{local region density}
\begin{equation}
\label{eq:local-density-deterministic-general}
\rho^\circ(f;x_0,r)
:=
\frac{N_{\mathrm{reg}}(f,B_\infty(x_0,r))}{\operatorname{vol}_{D_0}(B_\infty(x_0,r))}
=
\frac{N_{\mathrm{reg}}(f,B_\infty(x_0,r))}{(2r)^{D_0}}
\end{equation}
differs from the count by a constant factor, and its expectation $\rho(f;x_0,r):=\mathbb E[\rho^\circ(f;x_0,r)]$ is taken over the sources of randomness listed in Section~\ref{sec:bn}. Our main object is the centroid-centered case $x_0=\bar x$.

In the low-dimensional experiments we compute local region counts exactly by enumerating all affine regions that meet the window. At epoch $t$ we write $N_{\mathrm{reg},t}(\Omega):=N_{\mathrm{reg}}(f(\cdot;\theta_t),\Omega)$ for a baseline model and $N_{\mathrm{reg},t}^{\mathrm{BN}}(\Omega):=N_{\mathrm{reg}}(f^{\mathrm{inf}}(\cdot;\theta_t),\Omega)$ for a BN model, where $f^{\mathrm{inf}}$ is the network with BN in inference mode \eqref{eq:bn-infer}, that is, with the running statistics stored at epoch $t$. Both are region-intersection counts of CPA maps on the same window. The inference-mode map has the same form as the frozen-batch map \eqref{eq:bn-affine-reparam-train}, with the running statistics $(\bar\mu^{(l)},\bar v^{(l)})$ in place of the batch statistics. At initialization $\bar\mu^{(l)}=0$ and $\bar v^{(l)}=\mathbf 1$, so BN acts as the identity up to the factor $(1+\varepsilon)^{-1/2}$, whereas the frozen-batch map at initialization gives the central arrangement of Remark~\ref{rem:init-central}.

\section{Local Partition Geometry under Training-Time BN}
\label{sec:bn-local-density-standard-linf}

We study the local arrangement of switching hyperplanes produced by one hidden layer, in a baseline CPA network and in its training-time BN counterpart. The section has three parts of different status.
\begin{itemize}
\item \emph{Exact statements} (Sections~\ref{subsec:layerwise-density-linf}--\ref{subsec:exact-count}) hold for every parameter value and every reference batch: the form of the BN switching hyperplanes, an exact criterion for when a hyperplane cuts an $\ell_\infty$ window, and, under general position, an exact formula for the number of regions in a window.
\item \emph{Deterministic consequences} (Section~\ref{subsec:offset-control}) show that, under BN, cuts and interior intersections of switching hyperplanes are controlled by standardized offsets and by the spread of the batch.
\item \emph{A comparison of expected counts} (Section~\ref{subsubsec:probabilistic-single-layer}) is reduced to two explicit hypotheses, which are tested in Section~\ref{sec:experiments}.
\end{itemize}
Section~\ref{subsubsec:multilayer-extension} transfers the exact statements through depth.

\subsection{Hyperplane arrangements in representation space}
\label{subsec:layerwise-density-linf}

Fix a hidden layer $\ell$ and write
\begin{equation}
\label{eq:layerwise-shorthand}
\begin{gathered}
d:=D_{\ell-1},
\qquad
n:=D_\ell,
\qquad
Q:=K-1,\\
\mathcal J_\ell:=\{1,\dots,n\},
\qquad
\mathcal Q:=\{1,\dots,Q\},
\qquad
\Lambda_\ell:=\mathcal J_\ell\times\mathcal Q .
\end{gathered}
\end{equation}
Here $d$ is the dimension of the representation space feeding layer $\ell$, $n$ the width of layer $\ell$ and $Q$ the number of breakpoints. An element $a=(j,q)\in\Lambda_\ell$ indexes neuron $j$ and breakpoint $\tau_q$; for ReLU, $Q=1$ and we identify $a=(j,1)$ with $j$. Table~\ref{tab:layerwise-notation} collects the notation of this section.

\begin{table}[!tbp]
\caption{Notation for the analysis of a fixed hidden layer $\ell$.}
\label{tab:layerwise-notation}
\centering
\begin{tabular}{lp{0.66\linewidth}}
\toprule
Symbol & Meaning \\
\midrule
$d,\ n,\ Q$ & Input dimension of layer $\ell$, width of layer $\ell$, number of breakpoints \\
$\Lambda_\ell=\mathcal J_\ell\times\mathcal Q$ & Switching indices $a=(j,q)$ (neuron $j$, breakpoint $\tau_q$) \\
$\mathcal B$ & Reference mini-batch of inputs \\
$U=\{u^{(i)}\}_{i=1}^M$, $\bar u$ & Representations of $\mathcal B$ at the input of layer $\ell$, and their centroid \\
$B=B_\infty(\bar u,r)$ & $\ell_\infty$ window of radius $r$ \\
$g_j(u)=\langle w_j,u\rangle+b_j$ & Pre-activation of neuron $j$ \\
$H_a$, $H_a^{\mathrm{BN}}$ & Baseline and BN switching hyperplanes \\
$\mathcal A_\ell$, $\mathcal A_\ell^{\mathrm{BN}}$ & Baseline and BN arrangements of layer $\ell$ \\
$\delta_a=(\tau_q-\beta_j)/\gamma_j$ & Standardized BN offset \\
$H_j^\circ$ & Through-centroid reference hyperplane of neuron $j$ \\
$I_a(r)$, $I_a^{\mathrm{BN}}(r)$ & Window-cut indicators \\
$\Delta_a$, $\Delta_a^{\mathrm{BN}}$ & Normalized offsets; $I_a(r)=\mathbf 1\{\Delta_a<r\}$ \\
$f_k(\mathcal H,K)$ & Number of $k$-subsets of $\mathcal H$ whose intersection meets $K$ \\
$M_j(r)$ & Number of window-cutting hyperplanes of neuron $j$ \\
$s_{\max}(U)$ & Largest coordinatewise standard deviation of $U$ \\
$p_q(r)$, $p_q^{\mathrm{BN}}(r)$ & Window-cut probabilities (Section~\ref{subsubsec:probabilistic-single-layer}) \\
\bottomrule
\end{tabular}
\end{table}

\paragraph{Representation space.}
Let $\mathcal U_\ell:=h^{(\ell-1)}(\mathbb R^{D_0})\subseteq\mathbb R^d$ be the set of representations that reach layer $\ell$. For $\ell=1$, $\mathcal U_1=\mathbb R^{D_0}$. For $\ell\ge2$, $\mathcal U_\ell$ is a finite union of polyhedra of dimension at most $D_0$ and in general a proper subset of $\mathbb R^d$. Sections~\ref{subsec:layerwise-density-linf}--\ref{subsubsec:probabilistic-single-layer} study arrangements of hyperplanes in $\mathbb R^d$; for $\ell\ge2$, their consequences for the partition of the input space are derived in Section~\ref{subsubsec:multilayer-extension}, where windows are restricted to the image of a parent region.

\paragraph{Batch and window.}
Let $U:=\{u^{(1)},\dots,u^{(M)}\}$ with $u^{(i)}:=h_i^{(\ell-1)}$ be the representations of the reference mini-batch $\mathcal B$ at the input of layer $\ell$ (for BN networks, computed with the frozen-batch map), and let
\begin{equation}
\label{eq:batch-centroid-u}
\bar u:=\frac{1}{M}\sum_{i=1}^M u^{(i)},
\qquad
B:=B_\infty(\bar u,r),\quad r>0.
\end{equation}

\paragraph{Switching hyperplanes.}
For neuron $j\in\mathcal J_\ell$ let
\begin{equation}
\label{eq:preactivation-aj}
g_j(u):=\langle w_j,u\rangle+b_j,
\qquad
w_j\in\mathbb R^d\setminus\{0\},\ \ b_j\in\mathbb R,
\end{equation}
where the layer superscript is omitted. For $a=(j,q)\in\Lambda_\ell$, the baseline breakpoint-switching hyperplane and the baseline arrangement are
\begin{equation}
\label{eq:arrangement-unified}
H_a\equiv H_{j,q}:=\{u\in\mathbb R^d:g_j(u)=\tau_q\},
\qquad
\mathcal A_\ell:=\{H_a:a\in\Lambda_\ell\}.
\end{equation}

\paragraph{Local counts.}
For a finite family $\mathcal H$ of hyperplanes of $\mathbb R^d$, let $\mathcal R(\mathcal H)$ be the set of connected components (cells) of $\mathbb R^d\setminus\bigcup_{H\in\mathcal H}H$, and let
\begin{equation}
\label{eq:layerwise-region-count}
N_{\mathrm{reg}}(\mathcal H,B)
:=
\#\Bigl\{\text{connected components of }\operatorname{int}(B)\setminus \textstyle\bigcup_{H\in\mathcal H}H\Bigr\}.
\end{equation}
Every cell of $\mathcal R(\mathcal H)$ is open and convex. Since $B$ is convex with nonempty interior, a cell meets $B$ if and only if it meets $\operatorname{int}(B)$, and it then meets $\operatorname{int}(B)$ in a convex, hence connected, set. Therefore $N_{\mathrm{reg}}(\mathcal H,B)=\#\{R\in\mathcal R(\mathcal H):R\cap B\neq\varnothing\}$. A hyperplane is a \emph{window cut} if it meets $\operatorname{int}(B)$; contacts with the boundary only do not create new components. The local density is
\begin{equation}
\label{eq:layerwise-region-density}
\rho(\mathcal H;B)
:=
\frac{N_{\mathrm{reg}}(\mathcal H,B)}{\operatorname{vol}_d(B)}
=
\frac{N_{\mathrm{reg}}(\mathcal H,B)}{(2r)^d}.
\end{equation}

\begin{remark}[Layerwise viewpoint]
\label{rem:layerwise-viewpoint}
The preimage in input space of a switching hyperplane of a deep layer is in general a piecewise-affine hypersurface. Working in representation space keeps the switching sets affine and allows a standard hyperplane-arrangement analysis~\citep{book31,book32}.
\end{remark}

\subsection{Switching hyperplanes under training-time BN}
\label{subsec:bn-hyperplanes}

The batch statistics of $g_j$ on $U$ are
\begin{equation}
\label{eq:bn-stats-scalar}
\mu_j:=\frac{1}{M}\sum_{i=1}^M g_j(u^{(i)}),
\qquad
v_j:=\frac{1}{M}\sum_{i=1}^M \bigl(g_j(u^{(i)})-\mu_j\bigr)^2,
\end{equation}
and training-time BN transforms $g_j$ into
\begin{equation}
\label{eq:bn-normalized-aj}
\widehat g_j(u):=\gamma_j\,\frac{g_j(u)-\mu_j}{\sqrt{v_j+\varepsilon}}+\beta_j,
\qquad \varepsilon>0,\quad \gamma_j\neq0 .
\end{equation}

\begin{lemma}[Breakpoint-switching hyperplanes under BN]
\label{lem:bn-hyperplane-exact}
For each $a=(j,q)\in\Lambda_\ell$, the BN switching set $\{u:\widehat g_j(u)=\tau_q\}$ is the affine hyperplane
\begin{equation}
\label{eq:bn-hyperplane-unified}
H_a^{\mathrm{BN}}
=
\Bigl\{
u:\langle w_j,u\rangle=\langle w_j,\bar u\rangle+\delta_a\sqrt{v_j+\varepsilon}
\Bigr\},
\qquad
\delta_a\equiv\delta_{j,q}:=\frac{\tau_q-\beta_j}{\gamma_j}.
\end{equation}
Moreover $v_j=w_j^\top\Sigma_Uw_j$, where $\Sigma_U:=\frac1M\sum_i(u^{(i)}-\bar u)(u^{(i)}-\bar u)^\top$, so neither $v_j$ nor $H_a^{\mathrm{BN}}$ depends on the raw bias $b_j$.
\end{lemma}

\begin{proof}
By \eqref{eq:batch-centroid-u}, $\mu_j=\langle w_j,\bar u\rangle+b_j$, and therefore $g_j(u^{(i)})-\mu_j=\langle w_j,u^{(i)}-\bar u\rangle$, which gives $v_j=w_j^\top\Sigma_Uw_j$. The equation $\widehat g_j(u)=\tau_q$ is equivalent to $g_j(u)=\mu_j+\delta_a\sqrt{v_j+\varepsilon}$; substituting $g_j(u)=\langle w_j,u\rangle+b_j$ and the expression for $\mu_j$ yields \eqref{eq:bn-hyperplane-unified}.
\end{proof}

The BN arrangement of layer $\ell$ is
\begin{equation}
\label{eq:arrangement-bn}
\mathcal A_\ell^{\mathrm{BN}}:=\{H_a^{\mathrm{BN}}:a\in\Lambda_\ell\}.
\end{equation}

\begin{remark}[The bias is redundant under BN]
\label{rem:bias-redundant}
Lemma~\ref{lem:bn-hyperplane-exact} reflects the fact that a BN layer cancels its bias exactly: the gradient of the loss with respect to $b_j$ vanishes, which is why the bias is often omitted in practice. Its role as a location parameter is taken over by $\beta_j$, which however acts in batch-standardized units through $\delta_a$.
\end{remark}

\begin{corollary}[Through-centroid reference hyperplane]
\label{cor:through-centroid}
Let $H_j^\circ:=\{u:\langle w_j,u\rangle=\langle w_j,\bar u\rangle\}$, the zero set of the standardized coordinate $(g_j-\mu_j)/\sqrt{v_j+\varepsilon}$. Then $\bar u\in H_j^\circ$ for every neuron $j$, and each $H_a^{\mathrm{BN}}$, $a=(j,q)$, is the translate of $H_j^\circ$ at Euclidean distance $|\delta_a|\sqrt{v_j+\varepsilon}/\|w_j\|_2$. In particular, for ReLU and $\beta_j=0$, $H_{(j,1)}^{\mathrm{BN}}=H_j^\circ$, so all switching hyperplanes of the layer pass through $\bar u$ and $\mathcal A^{\mathrm{BN}}_\ell$ is a central arrangement. This last statement is the observation of \citet[Appendix~D.2]{book30}, who analyze BN with $\gamma=\mathbf 1$ and $\beta=0$.
\end{corollary}

\begin{proof}
Immediate from \eqref{eq:bn-hyperplane-unified} and the formula for the distance between parallel hyperplanes.
\end{proof}

Once $\beta$ is learned, the BN hyperplanes are no longer central: each is moved away from $\bar u$ by an amount proportional to the standardized offset $\delta_a$ and to the spread $\sqrt{v_j+\varepsilon}$ of the batch along $w_j$. The rest of this section studies the consequences of this structure for local counts.

\subsection{Exact window-cut criterion for \texorpdfstring{$\ell_\infty$}{linf} windows}
\label{subsec:intersection}

\begin{lemma}[$\ell_\infty$ window-cut criterion]
\label{lem:linf-intersection}
Let $w\in\mathbb R^d\setminus\{0\}$, $c\in\mathbb R$ and $B=B_\infty(\bar u,r)$. Then
\begin{equation}
\label{eq:linf-intersection}
\{u:\langle w,u\rangle=c\}\cap \operatorname{int}(B)\neq\varnothing
\iff
|c-\langle w,\bar u\rangle|<r\|w\|_1,
\end{equation}
and the same holds for the closed window $B$ with $\le$ in place of $<$. Moreover, for $H:=\{u:\langle w,u\rangle=c\}$,
\begin{equation}
\label{eq:l2-distance-hyperplane}
\mathrm{dist}_2(\bar u,H)=\frac{|c-\langle w,\bar u\rangle|}{\|w\|_2}.
\end{equation}
\end{lemma}

\begin{proof}
Write $u=\bar u+v$. By H\"older duality, $\{\langle w,v\rangle:\|v\|_\infty\le r\}=[-r\|w\|_1,r\|w\|_1]$ and, since $w\neq0$, $\{\langle w,v\rangle:\|v\|_\infty<r\}=(-r\|w\|_1,r\|w\|_1)$. Hence $\langle w,u\rangle=c$ has a solution in $B$ (respectively in $\operatorname{int}(B)$) if and only if $c-\langle w,\bar u\rangle$ lies in the closed (respectively open) interval. Equation~\eqref{eq:l2-distance-hyperplane} is the Euclidean distance formula.
\end{proof}

For $a=(j,q)\in\Lambda_\ell$ define the window-cut indicators
\begin{equation}
\label{eq:indicators-unified}
I_a(r):=\mathbf 1\{H_a\cap \operatorname{int}(B)\neq\varnothing\},
\qquad
I_a^{\mathrm{BN}}(r):=\mathbf 1\{H_a^{\mathrm{BN}}\cap \operatorname{int}(B)\neq\varnothing\}.
\end{equation}
By Lemmas~\ref{lem:bn-hyperplane-exact} and~\ref{lem:linf-intersection},
\begin{equation}
\label{eq:indicator-forms-unified}
\begin{aligned}
I_a(r)
&=
\mathbf 1\Bigl\{\bigl|\tau_q-(\langle w_j,\bar u\rangle+b_j)\bigr|<r\|w_j\|_1\Bigr\},\\
I_a^{\mathrm{BN}}(r)
&=
\mathbf 1\Bigl\{|\delta_a|\,\sqrt{v_j+\varepsilon}<r\|w_j\|_1\Bigr\}.
\end{aligned}
\end{equation}
Equivalently, $I_a(r)=\mathbf 1\{\Delta_a<r\}$ and $I_a^{\mathrm{BN}}(r)=\mathbf 1\{\Delta_a^{\mathrm{BN}}<r\}$ with the \emph{normalized offsets}
\begin{equation}
\label{eq:normalized-offsets-unified}
\Delta_a:=\frac{|\tau_q-(\langle w_j,\bar u\rangle+b_j)|}{\|w_j\|_1},
\qquad
\Delta_a^{\mathrm{BN}}:=\frac{|\delta_a|\sqrt{v_j+\varepsilon}}{\|w_j\|_1}.
\end{equation}

\begin{remark}[General window centers]
\label{rem:general-window-centers}
For a window $B_\infty(u_0,r)$ with arbitrary center $u_0\in\mathbb R^d$,
\begin{align*}
I_a(u_0,r)
&=
\mathbf 1\Bigl\{|\tau_q-(\langle w_j,u_0\rangle+b_j)|<r\|w_j\|_1\Bigr\},\\
I_a^{\mathrm{BN}}(u_0,r)
&=
\mathbf 1\Bigl\{|\langle w_j,\bar u-u_0\rangle+\delta_a\sqrt{v_j+\varepsilon}|<r\|w_j\|_1\Bigr\}.
\end{align*}
\end{remark}

\begin{remark}[Location versus spread]
\label{rem:location-vs-spread}
The two cut events depend on the batch in different ways. The baseline event depends on $U$ only through its centroid $\bar u$, together with the raw bias $b_j$: translating the batch or shifting the bias changes it. The BN event depends on $U$ only through the spread $v_j=w_j^\top\Sigma_Uw_j$ along $w_j$, together with the standardized offset $\delta_a$: it is invariant to translations of the batch and of the bias and, up to the $\varepsilon$ term, to positive rescaling of $w_j$. Section~\ref{subsec:offset-control} turns this observation into quantitative statements.
\end{remark}

\begin{remark}[Same parameters versus trained models]
\label{rem:same-parameters}
The identities above hold for every parameter value. Comparing $I_a$ and $I_a^{\mathrm{BN}}$ at the same $(W,b)$ therefore describes the effect of switching BN on at fixed weights, which is not what happens in training: a BN model and a baseline model trained from the same initialization reach different parameters. All comparison statements below are distributional; each expectation is taken over the training randomness of the corresponding model, and no parameters are assumed to be shared between the two models.
\end{remark}

\subsection{Exact local region counts}
\label{subsec:exact-count}

Let $\mathcal H$ be a finite family of distinct hyperplanes of $\mathbb R^d$ and $K\subseteq\mathbb R^d$ a nonempty open convex set. We say that $\mathcal H$ is \emph{in general position in $K$} if
\begin{itemize}
\item[(G1)] for every $S\subseteq\mathcal H$ with $|S|\le d$ and $\bigcap S\cap K\neq\varnothing$, the normals of the members of $S$ are linearly independent, and
\item[(G2)] no $d+1$ members of $\mathcal H$ have a common point in $K$.
\end{itemize}
For $k\ge1$ let
\begin{equation}
\label{eq:fk}
f_k(\mathcal H,K):=\#\Bigl\{S\subseteq\mathcal H:\ |S|=k,\ \textstyle\bigcap S\cap K\neq\varnothing\Bigr\},
\qquad f_0(\mathcal H,K):=1 .
\end{equation}
Thus $f_1$ is the number of window cuts and $f_k$, $k\ge2$, the number of $k$-wise intersections inside $K$. For $K=\mathbb R^d$ the following lemma reduces to the classical count $\sum_{k\le d}\binom{|\mathcal H|}{k}$ for arrangements in general position~\citep{book31,book32}; we include a proof because $K$ is a convex subset rather than the whole space.

\begin{lemma}[Exact local count]
\label{lem:exact-local-count}
If $\mathcal H$ is in general position in $K$, then
\begin{equation}
\label{eq:exact-local-count}
N(\mathcal H,K):=\#\Bigl\{\text{connected components of }K\setminus\textstyle\bigcup_{H\in\mathcal H}H\Bigr\}
=\sum_{k=0}^{d}f_k(\mathcal H,K).
\end{equation}
In particular, $N_{\mathrm{reg}}(\mathcal H,B)=\sum_{k=0}^d f_k(\mathcal H,\operatorname{int}(B))$.
\end{lemma}

\begin{proof}
By induction on $d$ and, for fixed $d$, on $|\mathcal H|$. If $\mathcal H=\varnothing$ both sides equal one. For $d=1$ the members of $\mathcal H$ that meet $K$ are distinct points of an open interval by (G2), and $N=1+f_1$.

Let $d\ge2$, let $\mathcal H=\mathcal H'\cup\{H\}$ with $H\notin\mathcal H'$, and assume \eqref{eq:exact-local-count} for $\mathcal H'$; general position is inherited by subfamilies. If $H\cap K=\varnothing$, neither side changes. Otherwise, every component $C$ of $K\setminus\bigcup\mathcal H'$ is open and convex, so $C\setminus H$ has two components if $H$ meets $C$ and one otherwise. The sets $C\cap H$ with $C\cap H\neq\varnothing$ are disjoint, open in $H$, convex and cover $(H\cap K)\setminus\bigcup_{H'\in\mathcal H'}H'$, so they are exactly its components. Hence the count increases by the number of components of $(H\cap K)\setminus\bigcup_{H'\in\mathcal H'}(H'\cap H)$.

Identify $H$ with $\mathbb R^{d-1}$. For $H'\in\mathcal H'$, the trace $H'\cap H$ is empty if $H'$ is parallel to $H$ and a hyperplane of $H$ otherwise. Traces that meet $H\cap K$ are pairwise distinct: if $H'_1\cap H=H'_2\cap H$ met $K$, then for $d\ge3$ the three normals would lie in a two-dimensional space, contradicting (G1), and for $d=2$ three lines would share a point of $K$, contradicting (G2). By (G1) and (G2) for $\mathcal H$, the traces meeting $H\cap K$ are in general position in the open convex set $H\cap K\subseteq H$. By the induction hypothesis in dimension $d-1$, the number of components equals the number of sets $S'$ of at most $d-1$ traces whose intersection meets $H\cap K$, which is the number of sets $S\ni H$ with $|S|\le d$ and $\bigcap S\cap K\neq\varnothing$. This is the increase of the right-hand side of \eqref{eq:exact-local-count}. The last claim follows from \eqref{eq:layerwise-region-count}.
\end{proof}

For a CPA activation with several breakpoints, the hyperplanes of one neuron are parallel. Let
\begin{equation}
\label{eq:familywise-counts}
M_j(r):=\sum_{q\in\mathcal Q}I_{j,q}(r),
\qquad
M_j^{\mathrm{BN}}(r):=\sum_{q\in\mathcal Q}I_{j,q}^{\mathrm{BN}}(r),
\qquad j\in\mathcal J_\ell,
\end{equation}
be the numbers of window-cutting hyperplanes of neuron $j$, and $\mathbf M(r):=(M_1(r),\dots,M_n(r))$, $\mathbf M^{\mathrm{BN}}(r):=(M_1^{\mathrm{BN}}(r),\dots,M_n^{\mathrm{BN}}(r))$.

\begin{corollary}[Bounds and the product formula]
\label{cor:bounds-product}
Let $\mathcal H\in\{\mathcal A_\ell,\mathcal A_\ell^{\mathrm{BN}}\}$ be in general position in $\operatorname{int}(B)$, with family counts $\mathbf M=\mathbf M(r)$ or $\mathbf M^{\mathrm{BN}}(r)$. Then
\begin{equation}
\label{eq:bounds-product}
1+\sum_{j\in\mathcal J_\ell}M_j\ \le\ N_{\mathrm{reg}}(\mathcal H,B)\ \le\
R_d^{\parallel}(\mathbf M)
:=\sum_{\substack{S\subseteq\mathcal J_\ell\\|S|\le d}}\ \prod_{j\in S}M_j,
\end{equation}
with the empty product equal to one. The upper bound is attained if and only if every intersection of at most $d$ window-cutting hyperplanes belonging to distinct neurons meets $\operatorname{int}(B)$.
\end{corollary}

\begin{proof}
Two hyperplanes of the same neuron are parallel and do not intersect, so a set $S$ counted in $f_k$ contains at most one hyperplane per neuron, and each of its members cuts the window. Hence $f_k$ is at most the $k$th elementary symmetric polynomial of $\mathbf M$, with equality if and only if all such intersections meet $\operatorname{int}(B)$. Summing over $k$ and using Lemma~\ref{lem:exact-local-count} gives the upper bound; the lower bound follows from $f_k\ge0$ for $k\ge2$ and $f_1=\sum_jM_j$.
\end{proof}

The upper bound in \eqref{eq:bounds-product} counts every intersection of window-cutting hyperplanes of distinct neurons. It is attained only when all these intersections meet the window, which does not follow from the fact that each hyperplane cuts the window.

\subsection{Standardized offsets control cuts and interior intersections}
\label{subsec:offset-control}

Lemma~\ref{lem:exact-local-count} shows that local refinement depends on two kinds of events: hyperplanes cutting the window, and intersections of several hyperplanes lying in the window. We now give a sufficient condition for the second kind, and then specialize both to BN.

Any hyperplane $H=\{u:\langle w,u\rangle=c\}$ can be written relative to a center $u_0$ as $H=\{u:\langle\widehat w,u-u_0\rangle=\widehat s\}$ with $\widehat w:=w/\|w\|_1$ and $\widehat s:=(c-\langle w,u_0\rangle)/\|w\|_1$. For $u_0=\bar u$ and $H=H_a$ or $H_a^{\mathrm{BN}}$, $|\widehat s|$ equals $\Delta_a$ or $\Delta_a^{\mathrm{BN}}$. For a set $S$ of $k$ hyperplanes, written in this form with normals $\widehat w_a$ and offsets $\widehat s_a$, $a\in S$, let $\widehat W_S\in\mathbb R^{k\times d}$ have rows $\widehat w_a$ and let $\widehat s_S:=(\widehat s_a)_{a\in S}$; $\sigma_k(\cdot)$ denotes the $k$th largest singular value.

\begin{proposition}[Interior intersections from small offsets]
\label{prop:interior-intersection}
Let $S$ be a set of $k\le d$ hyperplanes with $\operatorname{rank}\widehat W_S=k$. If
\[
\|\widehat s_S\|_2<r\,\sigma_k(\widehat W_S),
\]
then the intersection of the members of $S$ meets $\operatorname{int}B_\infty(u_0,r)$.
\end{proposition}

\begin{proof}
Let $p:=u_0+\widehat W_S^{+}\widehat s_S$, where $\widehat W_S^{+}$ is the Moore--Penrose pseudoinverse. Since $\widehat W_S$ has full row rank, $\widehat W_S\widehat W_S^{+}=I_k$, so $\langle\widehat w_a,p-u_0\rangle=\widehat s_a$ for every member $a$ of $S$, that is, $p$ lies in all of them. Moreover $\|p-u_0\|_\infty\le\|\widehat W_S^{+}\widehat s_S\|_2\le\|\widehat s_S\|_2/\sigma_k(\widehat W_S)<r$.
\end{proof}

Two hyperplanes of the same neuron have proportional normals, so a set $S$ containing both never satisfies the rank condition, consistently with the fact that they do not intersect. The condition is stronger than necessary for $k=1$, where Lemma~\ref{lem:linf-intersection} gives the exact criterion $|\widehat s_a|<r$.

For BN, the normalized offsets have a simple bound. Let $s_k(U):=(\Sigma_U)_{kk}^{1/2}$ be the standard deviation of the $k$th coordinate of the batch and $s_{\max}(U):=\max_{1\le k\le d}s_k(U)$.

\begin{proposition}[Standardized control of BN cuts and intersections]
\label{prop:bn-standardized-control}
Fix $\ell$, $U$ and $B=B_\infty(\bar u,r)$. For every $a=(j,q)\in\Lambda_\ell$,
\begin{equation}
\label{eq:kappa-bound}
\Delta_a^{\mathrm{BN}}=|\delta_a|\,\kappa_j,
\qquad
\kappa_j:=\frac{\sqrt{v_j+\varepsilon}}{\|w_j\|_1}
\ \le\ \bar\kappa_j:=\Bigl(s_{\max}(U)^2+\frac{\varepsilon}{\|w_j\|_1^2}\Bigr)^{1/2}.
\end{equation}
Consequently:
\begin{enumerate}
\item[(i)] $H_a^{\mathrm{BN}}$ cuts $B$ whenever $|\delta_a|\,\bar\kappa_j<r$;
\item[(ii)] for $S\subseteq\Lambda_\ell$ with $|S|=k\le d$ and $\operatorname{rank}\widehat W_S=k$, the intersection $\bigcap_{a\in S}H_a^{\mathrm{BN}}$ meets $\operatorname{int}(B)$ whenever $\bigl\|(|\delta_a|\,\bar\kappa_{j(a)})_{a\in S}\bigr\|_2<r\,\sigma_k(\widehat W_S)$, where $j(a)$ is the neuron of $a$;
\item[(iii)] if $\mathcal A^{\mathrm{BN}}_\ell$ is in general position in $\operatorname{int}(B)$, then $N_{\mathrm{reg}}(\mathcal A_\ell^{\mathrm{BN}},B)\ge 1+\sum_{k=1}^d G_k$, where $G_1$ is the number of $a$ satisfying (i) and, for $k\ge2$, $G_k$ is the number of sets $S$ of size $k$ satisfying (ii).
\end{enumerate}
These conditions involve the network only through the standardized offsets $\delta_a$, the normal directions $\widehat w_a$ and, through the $\varepsilon$ term, the norms $\|w_j\|_1$; they involve the batch only through $s_{\max}(U)$, and they do not involve the biases or the location $\bar u$. In contrast, for the baseline layer and fixed $W$ and $U$, $H_a$ cuts $\operatorname{int}(B)$ if and only if $b_j$ lies in an open interval of length $2r\|w_j\|_1$ centered at $\tau_q-\langle w_j,\bar u\rangle$. In particular, for every $W$ and $U$ there are biases for which $N_{\mathrm{reg}}(\mathcal A_\ell,B)=1$, so no lower bound on $N_{\mathrm{reg}}(\mathcal A_\ell,B)$ other than the trivial bound $1$ holds uniformly over the biases.
\end{proposition}

\begin{proof}
Since $|(\Sigma_U)_{kl}|\le s_k(U)s_l(U)$,
\[
v_j=w_j^\top\Sigma_Uw_j\le\Bigl(\sum_k|w_{jk}|\,s_k(U)\Bigr)^2\le s_{\max}(U)^2\|w_j\|_1^2,
\]
which gives \eqref{eq:kappa-bound}. Item (i) follows from \eqref{eq:indicator-forms-unified}. For (ii), apply Proposition~\ref{prop:interior-intersection} with $u_0=\bar u$ and $|\widehat s_a|=\Delta^{\mathrm{BN}}_a\le|\delta_a|\bar\kappa_{j(a)}$. For (iii), (i) and (ii) give $f_k(\mathcal A^{\mathrm{BN}}_\ell,\operatorname{int}(B))\ge G_k$, and Lemma~\ref{lem:exact-local-count} gives the bound. The baseline statement is the first identity in \eqref{eq:indicator-forms-unified}; choosing every $b_j$ outside the union of its $Q$ intervals gives $f_1=0$ and hence a single region.
\end{proof}

\begin{corollary}[Windows that contain the batch]
\label{cor:batch-covering}
Let $r_U:=\max_i\|u^{(i)}-\bar u\|_\infty$. If $r\ge r_U$, then $s_{\max}(U)\le r$, and every BN switching hyperplane with $|\delta_a|<\bigl(1+\varepsilon/(r^2\|w_j\|_1^2)\bigr)^{-1/2}$ cuts $B$. Up to the $\varepsilon$ term, every switching hyperplane with $|\delta_a|<1$ cuts every $\ell_\infty$ window centered at $\bar u$ that contains the batch.
\end{corollary}

\begin{proof}
Each coordinate of $u^{(i)}-\bar u$ is bounded by $r_U$ in absolute value, so $s_k(U)\le r_U\le r$. Insert $s_{\max}(U)\le r$ into Proposition~\ref{prop:bn-standardized-control}(i).
\end{proof}

\begin{remark}[Initialization is a degenerate case]
\label{rem:init-central}
For ReLU with $\beta=0$, as at the standard initialization, all $\delta_a=0$, every $H_a^{\mathrm{BN}}$ passes through $\bar u$, and $\mathcal A_\ell^{\mathrm{BN}}$ is central (Corollary~\ref{cor:through-centroid}). It then violates (G2) whenever $n>d$, so Lemma~\ref{lem:exact-local-count} does not apply. Every cell of a central arrangement is a cone with apex $\bar u$ and therefore meets every window centered at $\bar u$; when any $d$ of the normals are linearly independent, the number of cells is $2\sum_{k=0}^{d-1}\binom{n-1}{k}$~\citep{book56}, that is $2n$ for $d=2$. This is smaller than the generic value $\sum_{k=0}^d\binom nk$, which equals $1+n+\binom n2$ for $d=2$. At initialization BN therefore produces the largest possible number of window cuts, but not a fine local partition. Once the $\beta_j$ are small but distinct, the arrangement is a small perturbation of the central one; it is then in general position for almost every parameter value, and by Proposition~\ref{prop:bn-standardized-control}(ii) every intersection of hyperplanes whose normals are not close to linearly dependent lies in the window, so the count can be close to its generic maximum. The comparison results below exclude the central configuration.
\end{remark}

\subsection{Comparison of expected local counts}
\label{subsubsec:probabilistic-single-layer}

We now compare a BN model and a baseline model trained independently with the same architecture and protocol. Probabilities and expectations are conditional on the reference input batch $\mathcal B$, which both models share; the remaining randomness is the initialization and the data order of each model (Section~\ref{sec:bn}). Each model has its own representations $U$, centroid $\bar u$ and window $B=B_\infty(\bar u,r)$ at layer $\ell$. For $\ell=1$ the two windows coincide; for $\ell\ge2$ they have the same radius but lie in different representation spaces. We write $N^{\mathrm{BN}}:=N_{\mathrm{reg}}(\mathcal A_\ell^{\mathrm{BN}},B)$, $N:=N_{\mathrm{reg}}(\mathcal A_\ell,B)$, and $f_k^{\mathrm{BN}}$, $f_k$ for the corresponding quantities \eqref{eq:fk} with $K=\operatorname{int}(B)$.

\begin{theorem}[Reduction of the comparison]
\label{thm:reduction}
Assume that, almost surely, $\mathcal A_\ell$ and $\mathcal A^{\mathrm{BN}}_\ell$ are in general position in the interior of their windows. Then
\begin{equation}
\label{eq:reduction}
\mathbb E[N\mid\mathcal B]=\sum_{k=0}^{d}\mathbb E[f_k\mid\mathcal B],
\qquad
\mathbb E[f_k\mid\mathcal B]=\sum_{\substack{S\subseteq\Lambda_\ell\\|S|=k}}\mathbb P\Bigl(\textstyle\bigcap_{a\in S}H_a\cap\operatorname{int}(B)\neq\varnothing\ \Big|\ \mathcal B\Bigr),
\end{equation}
and likewise for the BN model. Consequently, if $\mathbb E[f_k^{\mathrm{BN}}\mid\mathcal B]\ge\mathbb E[f_k\mid\mathcal B]$ for $k=1,\dots,d$, then $\mathbb E[N^{\mathrm{BN}}\mid\mathcal B]\ge\mathbb E[N\mid\mathcal B]$, and the same holds for the densities \eqref{eq:layerwise-region-density}. Sets $S$ that contain two hyperplanes of the same neuron contribute zero in both models.
\end{theorem}

\begin{proof}
Apply Lemma~\ref{lem:exact-local-count} to each realization and take conditional expectations. The densities differ from the counts by the common factor $(2r)^{-d}$.
\end{proof}

Theorem~\ref{thm:reduction} uses only linearity of expectation, so it requires no independence between hyperplanes; such independence would be implausible, because hyperplanes of one neuron share $w_j$ and all units are trained jointly. General position fails only on a Lebesgue-null set of parameters, so the assumption holds when the law of the trained parameters is absolutely continuous; it fails at the standard initialization (Remark~\ref{rem:init-central}).

The next lemma shows that cut probabilities do not depend on the neuron, so they can be estimated by pooling over neurons.

\begin{lemma}[Exchangeability of hidden units]
\label{lem:exchangeability}
Suppose that the initial parameters attached to the units of layer $\ell$ (rows of $W^{(\ell)}$, entries of $b^{(\ell)}$, $\gamma^{(\ell)}$, $\beta^{(\ell)}$, and columns of $W^{(\ell+1)}$) are independent and identically distributed across units, and that the training procedure is equivariant under permutations of these units, as gradient descent and Adam are. Then, conditionally on $\mathcal B$, the offsets $(\Delta_{(j,q)})_{j\in\mathcal J_\ell}$ are exchangeable for each $q$, and likewise for the BN model. In particular, the cut probabilities
\[
p_q(r):=\mathbb P(\Delta_{(j,q)}<r\mid\mathcal B),
\qquad
p_q^{\mathrm{BN}}(r):=\mathbb P(\Delta^{\mathrm{BN}}_{(j,q)}<r\mid\mathcal B)
\]
do not depend on $j$. For ReLU in $d=2$, Theorem~\ref{thm:reduction} becomes
\begin{equation}
\label{eq:relu-2d-expectation}
\mathbb E[N\mid\mathcal B]=1+n\,p(r)+\binom{n}{2}p_2(r),
\end{equation}
where $p(r)=p_1(r)$ and $p_2(r)$ is the probability that the switching lines of two given neurons intersect inside the window.
\end{lemma}

\begin{proof}
A permutation of the units of layer $\ell$, applied jointly to the parameters listed above, leaves the lower layers and hence $U$ unchanged. By equivariance it maps the training trajectory to the permuted trajectory, and since the initial law is invariant under the permutation, so is the joint law of $U$ and the offsets of layer $\ell$. Equation~\eqref{eq:relu-2d-expectation} is \eqref{eq:reduction} with $d=2$ and $Q=1$.
\end{proof}

The comparison rests on the following two hypotheses.

\begin{hypothesis}[Cut-rate ordering]
\label{hyp:H1}
At layer $\ell$ and radius $r$, $p_q^{\mathrm{BN}}(r)\ge p_q(r)$ for every breakpoint $q$; equivalently, $\mathbb E[f_1^{\mathrm{BN}}\mid\mathcal B]\ge\mathbb E[f_1\mid\mathcal B]$ holds breakpoint by breakpoint.
\end{hypothesis}

\begin{hypothesis}[Interior-intersection ordering]
\label{hyp:H2}
At layer $\ell$ and radius $r$, $\mathbb E[f_k^{\mathrm{BN}}\mid\mathcal B]\ge\mathbb E[f_k\mid\mathcal B]$ for $k=2,\dots,d$.
\end{hypothesis}

Under Hypotheses~\ref{hyp:H1} and~\ref{hyp:H2}, Theorem~\ref{thm:reduction} gives $\mathbb E[N^{\mathrm{BN}}\mid\mathcal B]\ge\mathbb E[N\mid\mathcal B]$.

\paragraph{Status of the hypotheses.}
We do not prove Hypotheses~\ref{hyp:H1} and~\ref{hyp:H2}. A proof would require a description of the joint training dynamics of $\beta_j/\gamma_j$ in BN networks and of $b_j+\langle w_j,\bar u\rangle$ in baseline networks, which is not available. The hypotheses are motivated in two ways. First, by Proposition~\ref{prop:bn-standardized-control}, BN cuts and BN interior intersections are guaranteed when the standardized offsets $|\delta_a|$ are small compared with $r/s_{\max}(U)$. Since $\beta$ and $\gamma$ start at $0$ and $\mathbf 1$, these offsets start at $|\tau_q|$ (zero for ReLU), and the guarantees remain in force as long as $|\delta_a|\,\bar\kappa_j$ stays below $r$. For a baseline layer the corresponding quantities depend on the absolute position of the biases relative to the data, and Proposition~\ref{prop:bn-standardized-control} shows that no bound uniform in the biases exists. Second, the offset distributions that we observed in trained networks motivated the formulation. By Lemma~\ref{lem:exchangeability}, Hypothesis~\ref{hyp:H1} compares two one-dimensional distributions, which the pooled offset CDFs of Section~\ref{sec:experiments} estimate directly. For $d=2$, Hypothesis~\ref{hyp:H2} reduces to $p_2^{\mathrm{BN}}(r)\ge p_2(r)$; Section~\ref{subsec:exp1} examines the corresponding ordering for networks evaluated with running statistics, through the exact decomposition \eqref{eq:relu-2d-expectation}. We have no direct test for $k\ge3$.

\subsection{Local transfer through parent affine regions}
\label{subsubsec:multilayer-extension}

For a hidden layer $\ell\ge2$, the arrangement of Section~\ref{subsec:layerwise-density-linf} lives in representation space. We now relate it to the partition of the input space inside a parent region. Fix a model (for BN, its frozen-batch map) and a parent region $R$ of $h^{(\ell-1)}$, with $h^{(\ell-1)}(x)=\widetilde A_Rx+\widetilde d_R$ on $R$ as in \eqref{eq:parent-affine-map}. Let $m_R:=\operatorname{rank}\widetilde A_R$, assume $m_R\ge1$, and let
\begin{equation}
\label{eq:image-subspace-SR}
S_R:=\widetilde d_R+\operatorname{Im}(\widetilde A_R)\subseteq\mathbb R^d
\end{equation}
be the affine hull of the image. We do not assume $m_R=D_0$: this fails whenever fewer than $D_0$ units are active along the prefix, and always when a layer is narrower than $D_0$. For a center $u_0\in h^{(\ell-1)}(R)$, for example the image of a data point in $R$, and $r>0$, define the image window and its preimage
\begin{equation}
\label{eq:window-in-image-subspace}
B_R:=\operatorname{int}B_\infty(u_0,r)\cap h^{(\ell-1)}(R),
\qquad
\Omega_R:=\bigl\{x\in R:\ h^{(\ell-1)}(x)\in\operatorname{int}B_\infty(u_0,r)\bigr\}.
\end{equation}
Since $R$ is open and convex and $h^{(\ell-1)}$ is affine on $R$, $B_R$ is a nonempty open convex subset of $S_R$ and $\Omega_R$ is an open convex subset of $\mathbb R^{D_0}$. The batch centroid $\bar u$ need not lie in $S_R$, because the examples of the batch belong to different parent regions; the cut criterion of Remark~\ref{rem:general-window-centers} applies to windows centered at $u_0$.

For $a\in\Lambda_\ell$, the pullback switching sets are $\widetilde H_{a,R}:=\{x\in R:h^{(\ell-1)}(x)\in H_a\}$ and $\widetilde H_{a,R}^{\mathrm{BN}}$, defined in the same way.

\begin{proposition}[Count transfer for arbitrary rank]
\label{prop:count-transfer}
Let $\mathcal H$ be a finite family of hyperplanes of $\mathbb R^d$, none of which contains $S_R$, and let $\widetilde{\mathcal H}_R:=\{\{x\in R:h^{(\ell-1)}(x)\in H\}:H\in\mathcal H\}$.
\begin{enumerate}
\item[(i)] For $H=\{u:\langle w,u\rangle=c\}\in\mathcal H$, the pullback is empty if $\widetilde A_R^\top w=0$, and otherwise it is the intersection of $R$ with an affine hyperplane of $\mathbb R^{D_0}$, which may be empty.
\item[(ii)] The number of connected components of $\Omega_R\setminus\bigcup\widetilde{\mathcal H}_R$ equals the number of connected components of $B_R\setminus\bigcup_{H\in\mathcal H}(H\cap S_R)$. If the traces $H\cap S_R$ are in general position in $B_R$, computed in the $m_R$-dimensional space $S_R$, this number is $\sum_{k=0}^{m_R}f_k(\{H\cap S_R\},B_R)$.
\item[(iii)] If $m_R=D_0$, then $\operatorname{vol}_{D_0}(\Omega_R)=\det(\widetilde A_R^\top\widetilde A_R)^{-1/2}\operatorname{vol}^{S_R}_{D_0}(B_R)$, so densities in $\Omega_R$ and in $B_R$ differ by a constant factor that depends only on $\widetilde A_R$.
\end{enumerate}
\end{proposition}

\begin{proof}
(i) On $R$, $\langle w,h^{(\ell-1)}(x)\rangle=c$ reads $\langle\widetilde A_R^\top w,x\rangle=c-\langle w,\widetilde d_R\rangle$. If $\widetilde A_R^\top w=0$, then $\langle w,\cdot\rangle$ is constant on $S_R$, so the pullback is either empty or all of $R$; the latter would mean $S_R\subseteq H$, which is excluded.

(ii) Let $\varphi:=h^{(\ell-1)}|_{\Omega_R}:\Omega_R\to B_R$. It is continuous and surjective by construction. It is open, because it is the restriction to an open set of the affine surjection $x\mapsto\widetilde A_Rx+\widetilde d_R$ of $\mathbb R^{D_0}$ onto $S_R$. Its fibres are intersections of affine subspaces with the convex set $\Omega_R$, hence connected. Let $V\subseteq B_R$ be open and connected, and suppose that $\varphi^{-1}(V)=A\cup A'$ with $A,A'$ disjoint, open and nonempty. Every fibre over $V$ is connected and therefore lies in $A$ or in $A'$, so $\varphi(A)$ and $\varphi(A')$ are disjoint, open and nonempty, and their union is $V$, a contradiction. Hence $\varphi^{-1}$ maps the components of $B_R\setminus\bigcup_H(H\cap S_R)$ to connected sets, which are disjoint, open and cover $\Omega_R\setminus\bigcup\widetilde{\mathcal H}_R$; they are therefore exactly its components. Since none of the $H$ contains $S_R$, each trace $H\cap S_R$ is empty or a hyperplane of $S_R$, and Lemma~\ref{lem:exact-local-count} applies in $S_R$.

(iii) If $m_R=D_0$, the map $x\mapsto\widetilde A_Rx+\widetilde d_R$ is injective with constant Jacobian $\det(\widetilde A_R^\top\widetilde A_R)^{1/2}$, and the identity is the area formula.
\end{proof}

Proposition~\ref{prop:count-transfer} is an exact statement about one model. Combined with Lemma~\ref{lem:exact-local-count}, it shows that local refinement in the input-space set $\Omega_R$ is governed by the same quantities as in Section~\ref{subsec:exact-count}, namely cuts and intersections of the restricted window $B_R$. Theorem~\ref{thm:reduction} therefore applies verbatim to the windows $B_R$, with $d$ replaced by $m_R$; a comparison between the two models then requires Hypotheses~\ref{hyp:H1}--\ref{hyp:H2} for these restricted windows. Since the two models have different parent regions, we do not formulate a separate comparison theorem for this setting.

\section{Experiments}
\label{sec:experiments}

We empirically examine the geometric mechanism developed in Section~\ref{sec:bn-local-density-standard-linf}. The experiments are organized at three levels. First, on low-dimensional problems where exact enumeration is feasible, we directly measure local affine-region counts and test the batch-conditional geometric statements appearing in the theory. Second, for deep networks, we separately evaluate local ingredients of the multilayer construction and the resulting global region counts. Third, on higher-dimensional datasets where exact counting is intractable, we use the normalized offsets, which determine the window cuts exactly, and low-dimensional slices.

\paragraph{Common protocol.}
Unless stated otherwise, BN is applied \emph{pre-activation}. All BN/non-BN comparisons use matched architectures, optimizers, learning-rate schedules, batch sizes, initialization protocols, and training horizons. Toy-data experiments are repeated over 10 random seeds, and real-data experiments over 5 seeds.

Exact region counts and partition plots (Sections~\ref{subsec:exp1}, \ref{subsec:exp2_multilayer} and~\ref{subsec:realdata_affine_regions_figure_only}) are computed for the inference-mode map of Section~\ref{sec:local-density}, with the running statistics stored at each checkpoint. The checks of the exact identities (Section~\ref{subsec:identity-checks}) and the offset distributions of Section~\ref{subsec:mechanism_validation_proximity_cutrate}, apart from the inference-mode cut rates described there, use the frozen-batch map: we fix reference mini-batches, use their batch statistics, and average over reference batches where stated. Quantitative entries are reported as mean$\pm$std over matched seeds unless otherwise specified.

BN and non-BN models are trained independently and reach different parameters (Remark~\ref{rem:same-parameters}). The experiments therefore test the distributional statements of Section~\ref{subsubsec:probabilistic-single-layer}, in particular Hypotheses~\ref{hyp:H1} and~\ref{hyp:H2}, while the exact identities of Sections~\ref{subsec:bn-hyperplanes}--\ref{subsec:intersection} are checked within each model (Section~\ref{subsec:identity-checks}).

\subsection{Exact Local Region Counts in Single-Layer Networks}
\label{subsec:exp1}

We begin with the quantity that appears most directly in the single-layer analysis: the exact number of affine regions intersecting a local $\ell_\infty$ window,
\[
N_{\mathrm{reg}}(\Omega), \qquad \Omega = B_\infty(\bar x,r)\subset \mathbb{R}^2.
\]
For a fixed window, comparing $N_{\mathrm{reg}}(\Omega)$ is equivalent to comparing local region density, since the denominator $(2r)^2$ is common.

We train single-hidden-layer ReLU networks with widths $h\in\{32,64,128\}$ for 100 epochs using Adam (learning rate $10^{-4}$, batch size $64$) with Kaiming-uniform initialization. For BN models, standard BN is inserted before the ReLU. At each checkpoint, we compute $N_{\mathrm{reg}}(\Omega)$ by exact region enumeration.

We consider three two-dimensional classification datasets, each containing 200 examples, shown in Figure~\ref{fig:toy_datasets_overview}:
\begin{enumerate}
    \item \textbf{Gaussian Quantiles (5-class):} $\bar x=(0,0)$ and $\Omega_{\mathrm{Gauss}}=[-1,1]^2$ ($r=1$).
    \item \textbf{Two Moons (2-class):} $\bar x\approx(0.5,0.5)$ and $\Omega_{\mathrm{Moon}}=[-1,2]^2$ ($r=1.5$).
    \item \textbf{Random Uniform (2-class):} $\bar x=(2,2)$ and $\Omega_{\mathrm{Rand}}=[-0.5,4.5]^2$ ($r=2.5$).
\end{enumerate}
The Random Uniform dataset is included to reduce dependence on a particular low-dimensional class-manifold structure.

\begin{figure}[!tbp]
    \centering
    \includegraphics[width=0.95\linewidth]{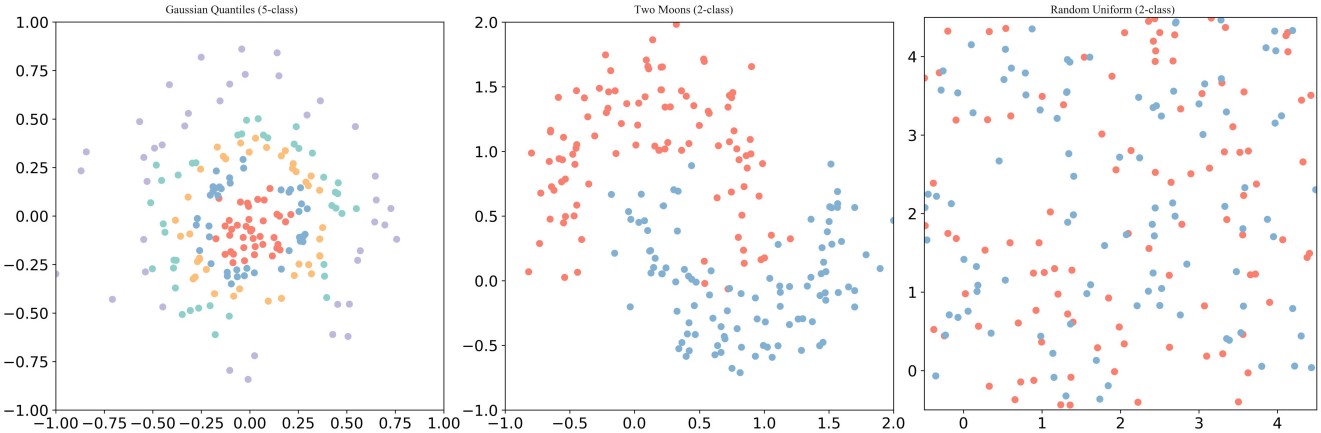}
    \caption{The three two-dimensional datasets used in the local-region experiments: Gaussian Quantiles, Two Moons, and Random Uniform. The plots illustrate the data geometries underlying the centroid-centered evaluation windows used in the experiments.}
    \label{fig:toy_datasets_overview}
\end{figure}

Table~\ref{tab:exp1_region_counts} reports the exact local region counts at epoch 100. Across all datasets and widths, BN models have larger local region counts inside the same centroid-anchored window, and the difference increases with width.

\begin{table}[!tbp]
\caption{Exact local region counts in single-layer networks at epoch 100. Entries are reported as mean$\pm$std over 10 matched random seeds.}
\label{tab:exp1_region_counts}
\centering
\begin{tabular}{llccc}
\toprule
\textbf{Dataset} & \textbf{Model} & $\mathbf{h=32}$ & $\mathbf{h=64}$ & $\mathbf{h=128}$ \\
\midrule
\multirow{2}{*}{Gaussian Quantiles}
& non-BN & $186 \pm 3$ & $507 \pm 8$ & $2,238 \pm 16$ \\
& BN     & $\mathbf{447 \pm 8}$ & $\mathbf{1,720 \pm 12}$ & $\mathbf{6,552 \pm 13}$ \\
\midrule
\multirow{2}{*}{Two Moons}
& non-BN & $273 \pm 5$ & $1,138 \pm 8$ & $3,968 \pm 16$ \\
& BN     & $\mathbf{433 \pm 2}$ & $\mathbf{1,712 \pm 12}$ & $\mathbf{6,784 \pm 6}$ \\
\midrule
\multirow{2}{*}{Random Uniform}
& non-BN & $267 \pm 1$ & $679 \pm 5$ & $2,953 \pm 8$ \\
& BN     & $\mathbf{452 \pm 4}$ & $\mathbf{1,668 \pm 7}$ & $\mathbf{5,830 \pm 4}$ \\
\bottomrule
\end{tabular}
\end{table}

\paragraph{Decomposition of the count difference.}
For a single hidden layer with input dimension two, Lemma~\ref{lem:exact-local-count} applies to the window $\Omega$ and gives $N_{\mathrm{reg}}(\Omega)=1+f_1+f_2$ whenever the switching lines are in general position, which is the case for almost every parameter value. Here $f_1\le n$ is the number of lines that cut $\Omega$ and $f_2\le\binom n2$ the number of pairs of lines that cross inside $\Omega$; in particular $N_{\mathrm{reg}}(\Omega)\le 1+n+\binom n2$. Since $0\le f_1\le n$ in both models, the difference in mean counts exceeds the difference in mean numbers of interior crossings by at most $n$:
\begin{equation}
\label{eq:f2-gap-bound}
\mathbb E\bigl[f_2^{\mathrm{BN}}\bigr]-\mathbb E[f_2]\ \ge\ \mathbb E\bigl[N_{\mathrm{reg}}^{\mathrm{BN}}(\Omega)\bigr]-\mathbb E\bigl[N_{\mathrm{reg}}(\Omega)\bigr]-n .
\end{equation}
Table~\ref{tab:count-decomposition} applies \eqref{eq:f2-gap-bound} to the means of Table~\ref{tab:exp1_region_counts}. In every configuration, additional window cuts can account for at most $20\%$ of the difference, so at least $80\%$ of it comes from intersections inside the window; for width $128$ the share is at least $95\%$. The BN models reach $71$--$85\%$ of the maximal count $1+n+\binom n2$, against $24$--$55\%$ for the non-BN models. Lemma~\ref{lem:exact-local-count} applies to any arrangement in general position, so this decomposition is exact for the inference-mode map used here. The observations support, at the level of the means, the ordering of interior intersections stated in Hypothesis~\ref{hyp:H2} in two dimensions, measured on the inference-mode arrangement. They are also consistent with Proposition~\ref{prop:bn-standardized-control}(ii) and Remark~\ref{rem:init-central}: when the standardized offsets are small, two switching lines intersect inside the window unless their normals are nearly parallel.

\begin{table}[!tbp]
\caption{Decomposition of the single-layer counts of Table~\ref{tab:exp1_region_counts} by Lemma~\ref{lem:exact-local-count}. $N_{\max}=1+n+\binom n2$ is the largest possible count for $n$ lines in the plane. The share columns give the mean count as a fraction of $N_{\max}$, and $\Delta N$ is the difference of mean counts between BN and non-BN models. By \eqref{eq:f2-gap-bound}, additional window cuts can account for at most the fraction $n/\Delta N$ of this difference.}
\label{tab:count-decomposition}
\centering
\begin{tabular}{lrrccrc}
\toprule
\textbf{Dataset} & $n$ & $N_{\max}$ & non-BN share & BN share & $\Delta N$ & $n/\Delta N$ \\
\midrule
\multirow{3}{*}{Gaussian Quantiles} & 32 & 529 & 35\% & 84\% & 261 & 12.3\% \\
 & 64 & 2,081 & 24\% & 83\% & 1,213 & 5.3\% \\
 & 128 & 8,257 & 27\% & 79\% & 4,314 & 3.0\% \\
\midrule
\multirow{3}{*}{Two Moons} & 32 & 529 & 52\% & 82\% & 160 & 20.0\% \\
 & 64 & 2,081 & 55\% & 82\% & 574 & 11.1\% \\
 & 128 & 8,257 & 48\% & 82\% & 2,816 & 4.5\% \\
\midrule
\multirow{3}{*}{Random Uniform} & 32 & 529 & 50\% & 85\% & 185 & 17.3\% \\
 & 64 & 2,081 & 33\% & 80\% & 989 & 6.5\% \\
 & 128 & 8,257 & 36\% & 71\% & 2,877 & 4.4\% \\
\bottomrule
\end{tabular}
\end{table}

\paragraph{Counts during training.}
We track $N_{\mathrm{reg},t}(\Omega)$ over epochs for the three windows. Figure~\ref{fig:exp1_regioncount_over_epochs} shows that the BN and non-BN counts are indistinguishable at initialization (epoch 0) and that the BN curves lie above the non-BN curves at all later checkpoints, for all widths and datasets. The coincidence at epoch 0 follows from the evaluation mode: at that checkpoint the running statistics still have their initial values, so BN acts as the identity, and matched seeds give both models the same initial weights. It is not the frozen-batch arrangement of Remark~\ref{rem:init-central}, which at initialization is central and has $2n$ regions in the window.
Figure~\ref{fig:exp1_visual_partition} shows a representative partition visualization for one of the two-dimensional tasks. The visible partition inside the evaluation window is qualitatively consistent with the count differences reported above.

\begin{figure}[!tbp]
    \centering
    \includegraphics[width=0.64\linewidth]{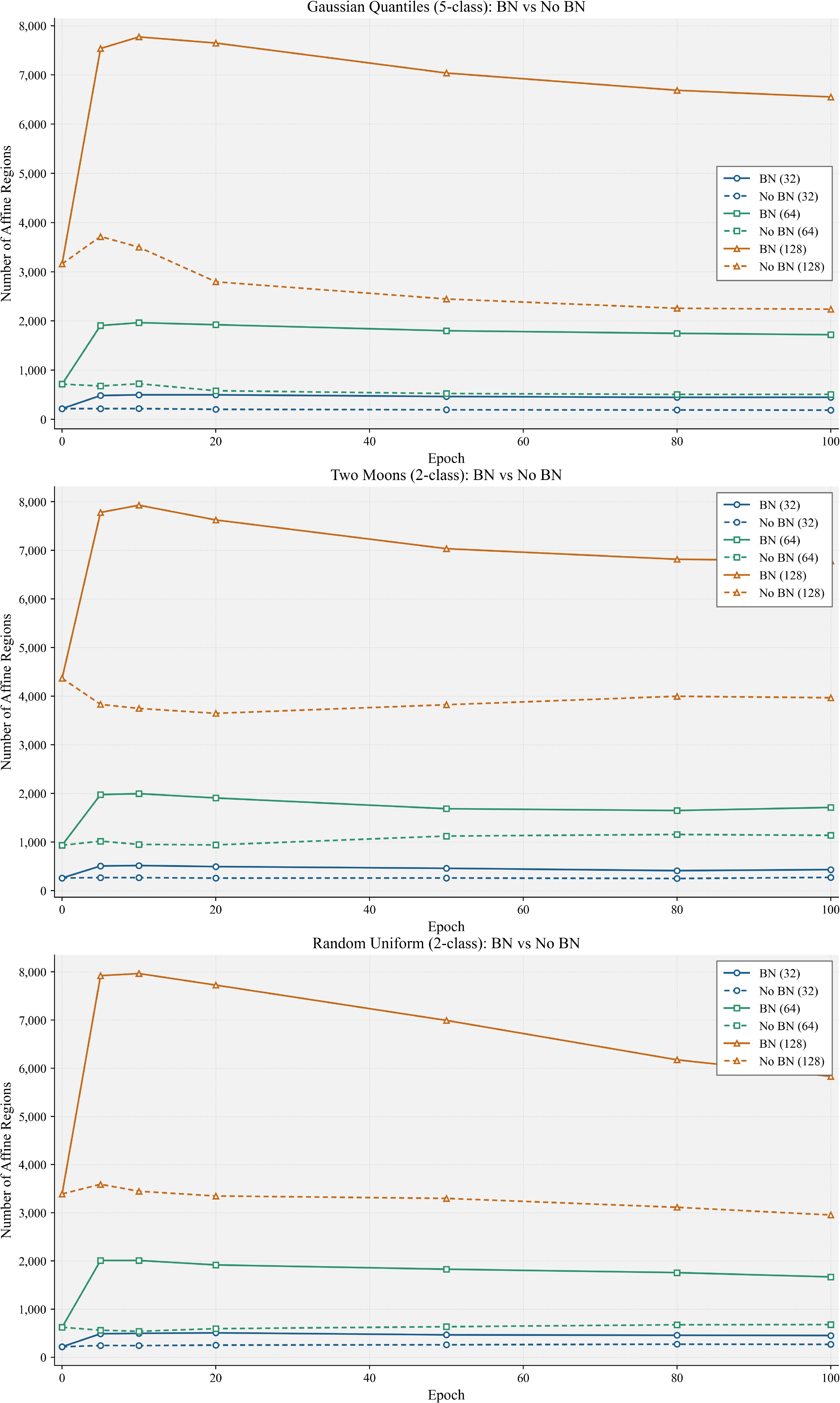}
    \caption{Training dynamics of exact local region counts in single-layer networks. We plot $N_{\mathrm{reg},t}(\Omega)$ over epochs for three window configurations under matched BN and non-BN training, with BN in inference mode (Section~\ref{sec:local-density}). At epoch 0 BN acts as the identity and the two counts coincide; at all later checkpoints the BN curves lie above the non-BN curves.}
    \label{fig:exp1_regioncount_over_epochs}
\end{figure}

\begin{figure}[!tbp]
    \centering
    \includegraphics[width=0.95\linewidth]{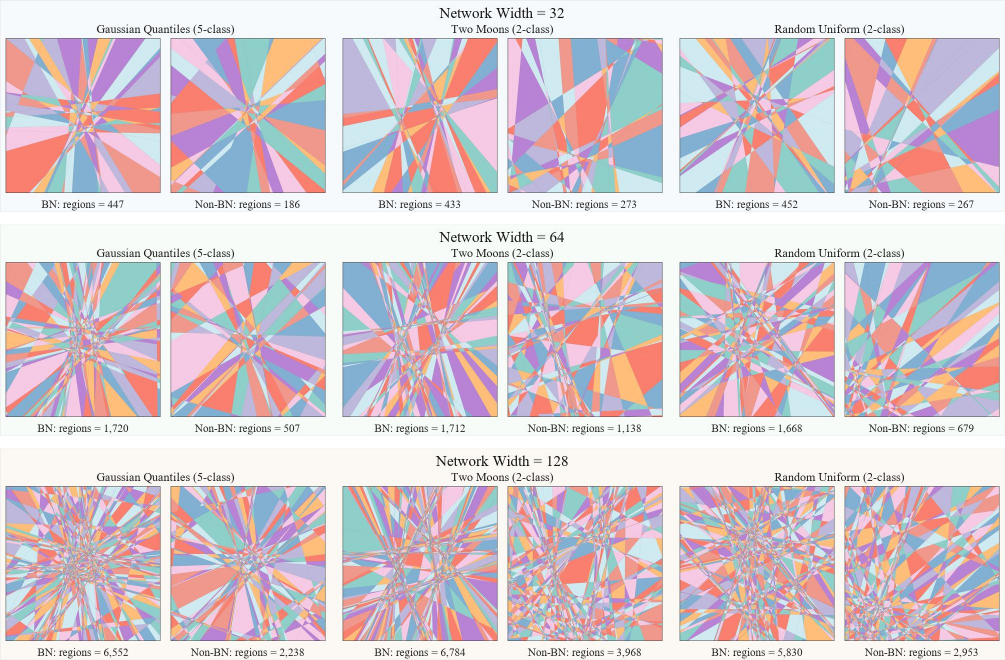}
    \caption{Representative single-layer partition visualization on a two-dimensional task. The evaluation window $\Omega$ is highlighted. The BN model induces a finer visible partition inside the displayed neighborhood, consistent with the exact counts reported above.}
    \label{fig:exp1_visual_partition}
\end{figure}

\subsection{Checks of the Exact Identities}
\label{subsec:identity-checks}

Before turning to the distributional statements, we check under fixed reference batches that the frozen-batch map and the cut test behave as stated in Sections~\ref{subsec:bn-hyperplanes}--\ref{subsec:intersection}. Since these statements are exact, agreement is expected; the checks confirm the implementation rather than test a hypothesis.

\paragraph{Bias invariance (Lemma~\ref{lem:bn-hyperplane-exact} and Remark~\ref{rem:bias-redundant}).}
Under BN the raw bias cancels, so the BN offsets do not depend on $b$, whereas the baseline offsets depend on $w^\top\bar u+b$. We check this on Two Moons and Gaussian Quantiles in two ways. First, we compute across neurons the Pearson correlation between the offset magnitude and $|b|$: it is clearly nonzero for non-BN models and much smaller for BN models (Figure~\ref{fig:bn_bias_corr_bars}). Second, holding $(W,\gamma,\beta)$ and the reference batch fixed, we shift the layer bias by $b\leftarrow b+c\mathbf 1$ and recompute the offsets: the non-BN offsets change, whereas the BN offsets are unchanged up to numerical precision (Figure~\ref{fig:bn_bias_perturb_sanity}).

\begin{figure}[!tbp]
  \centering
  \includegraphics[width=0.45\linewidth]{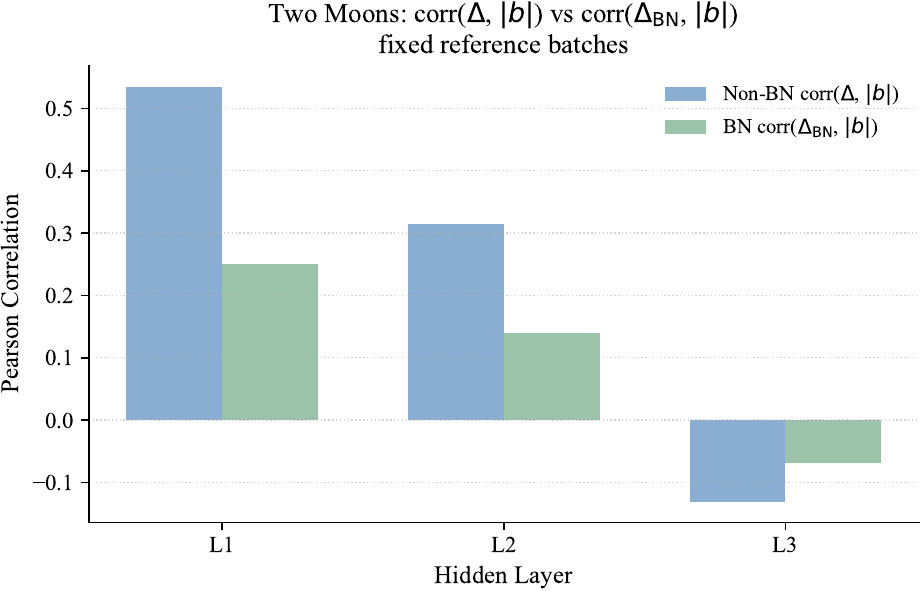}\hfill
  \includegraphics[width=0.45\linewidth]{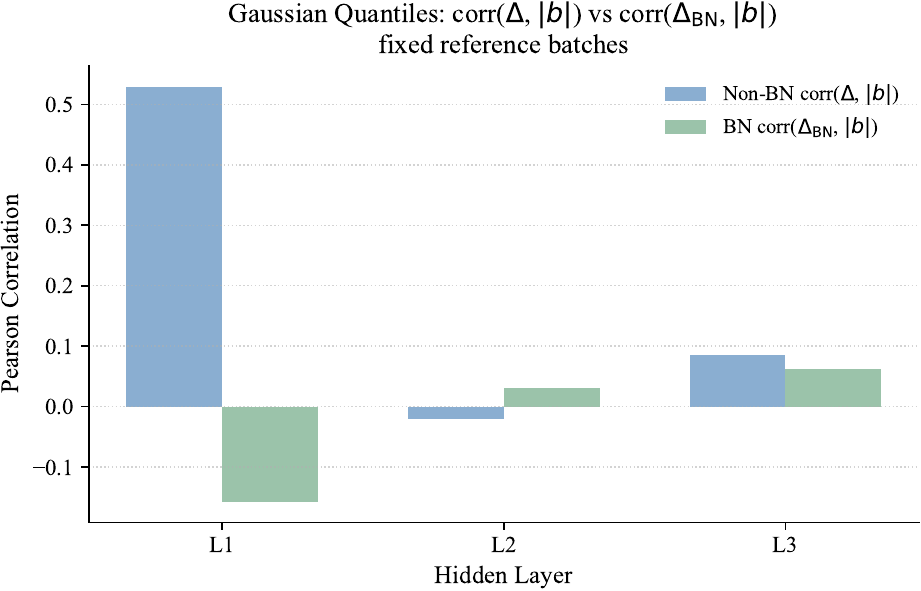}
  \caption{Bias-decoupling diagnostic under fixed reference batches. We plot layerwise Pearson correlations between normalized offsets and raw-bias magnitudes. Non-BN offsets depend more strongly on $|b|$, while BN correlations are substantially reduced.}
  \label{fig:bn_bias_corr_bars}
\end{figure}

\begin{figure}[!tbp]
  \centering
  \includegraphics[width=0.45\linewidth]{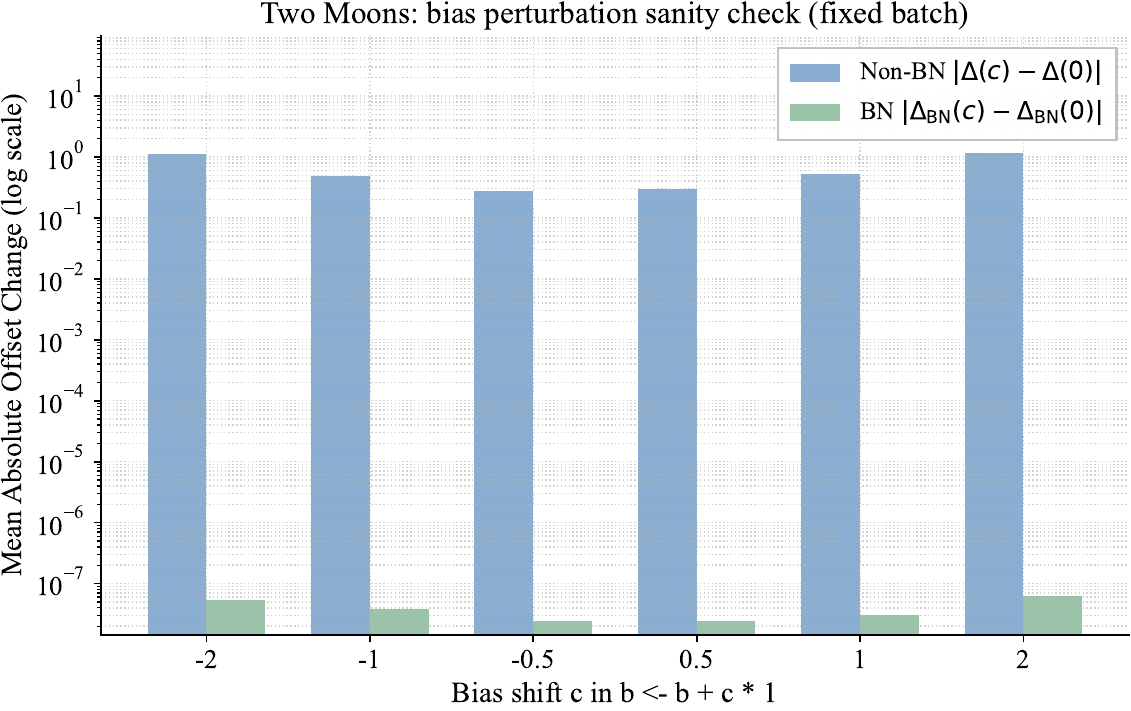}\hfill
  \includegraphics[width=0.45\linewidth]{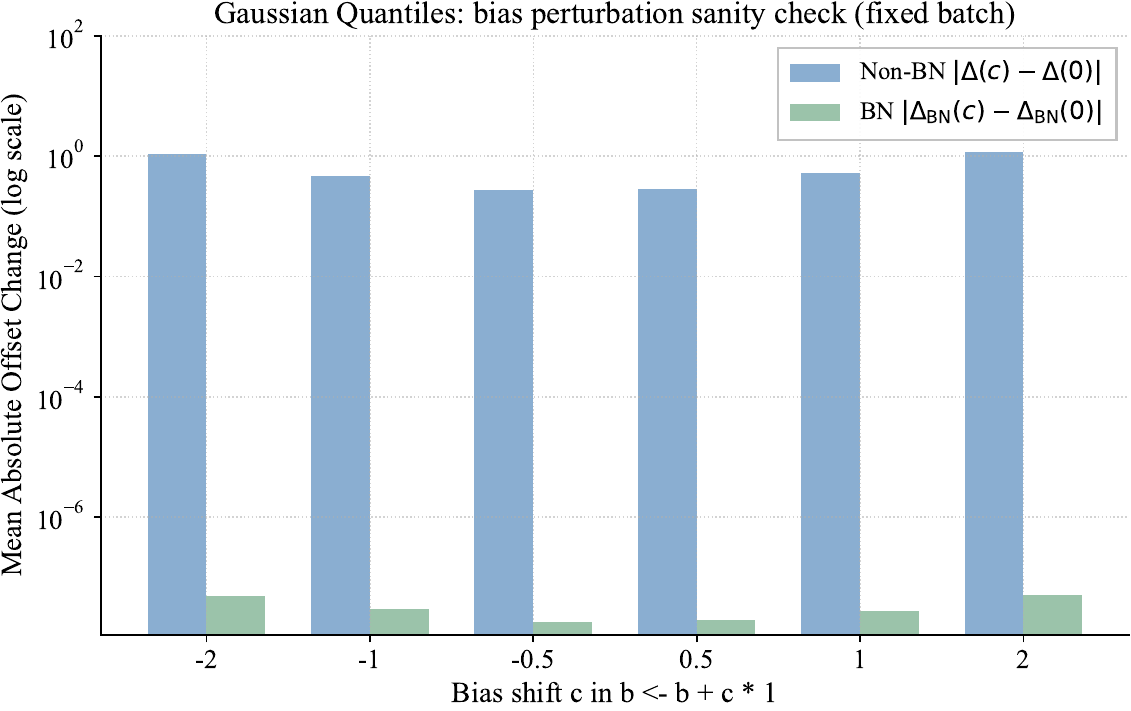}
  \caption{Explicit bias-shift invariance under fixed reference batches. After applying $b\leftarrow b+c\mathbf 1$ while holding $(W,\gamma,\beta)$ and $U$ fixed, non-BN offsets change substantially, whereas BN offsets remain invariant up to numerical precision.}
  \label{fig:bn_bias_perturb_sanity}
\end{figure}

\paragraph{Through-centroid structure (Corollary~\ref{cor:through-centroid}).}
For representative neurons of the first hidden layer, the reference hyperplanes $H_j^\circ:\langle w_j,u\rangle=\langle w_j,\bar u\rangle$ pass through the batch centroid and the BN switching hyperplanes are parallel translates of them, with residuals at numerical precision (Figure~\ref{fig:through_centroid_vs_bn_lines}).

\begin{figure}[!tbp]
  \centering
  \includegraphics[width=0.45\linewidth]{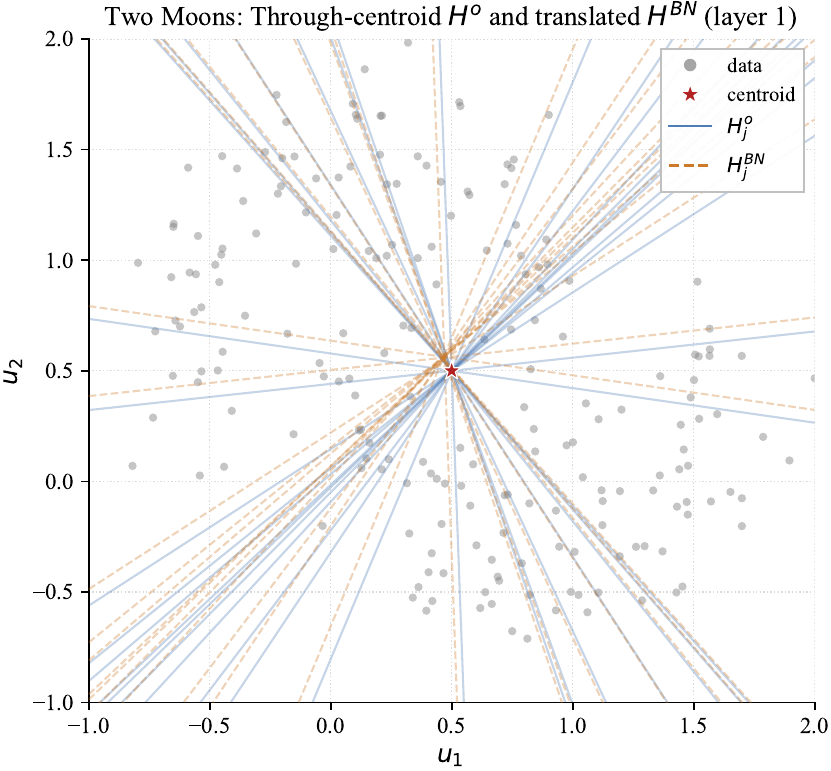}\hfill
  \includegraphics[width=0.45\linewidth]{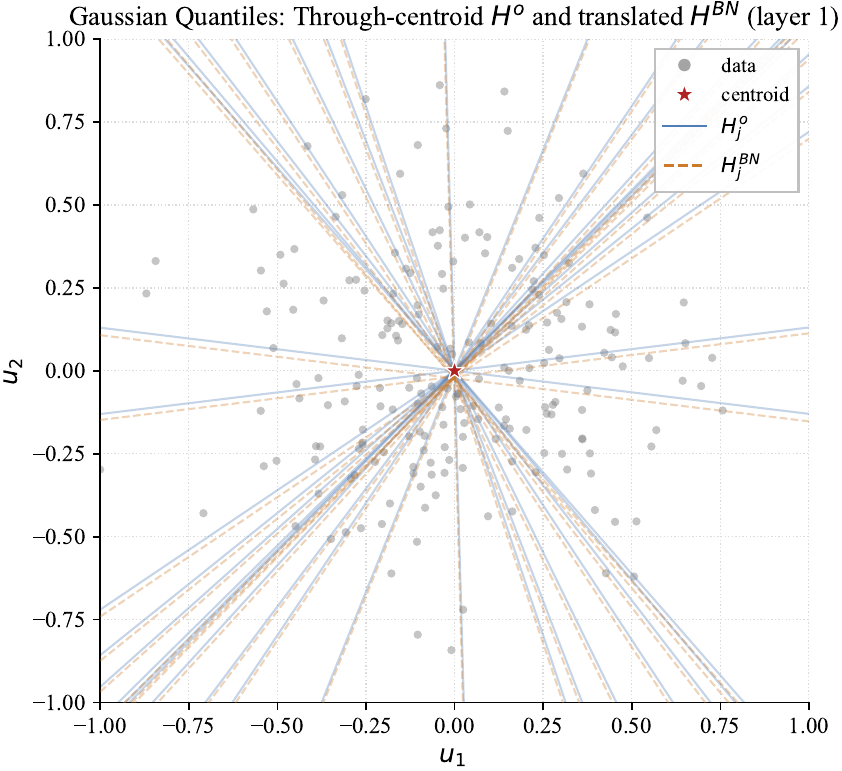}
  \caption{Training-time batch-conditional hyperplanes under a fixed reference batch. The through-centroid reference hyperplanes $H_j^\circ$ pass through the batch centroid $\bar u$, while the BN switching hyperplanes $H_j^{\mathrm{BN}}$ are parallel translates.}
  \label{fig:through_centroid_vs_bn_lines}
\end{figure}

\paragraph{Window-cut criterion (Lemma~\ref{lem:linf-intersection}).}
We compare two computations of the same event: the closed-form criterion $|c-\langle w,\bar u\rangle|<r\|w\|_1$ of Lemma~\ref{lem:linf-intersection}, and a direct geometric test of whether the hyperplane $\{u:\langle w,u\rangle=c\}$ meets the open box $\operatorname{int}B_\infty(\bar u,r)$. Since Lemma~\ref{lem:linf-intersection} is exact, the two must agree; Figure~\ref{fig:linf_equivalence_sanity} confirms that they agree for all neurons and radii evaluated.

\begin{figure}[!tbp]
  \centering
  \includegraphics[width=0.65\linewidth]{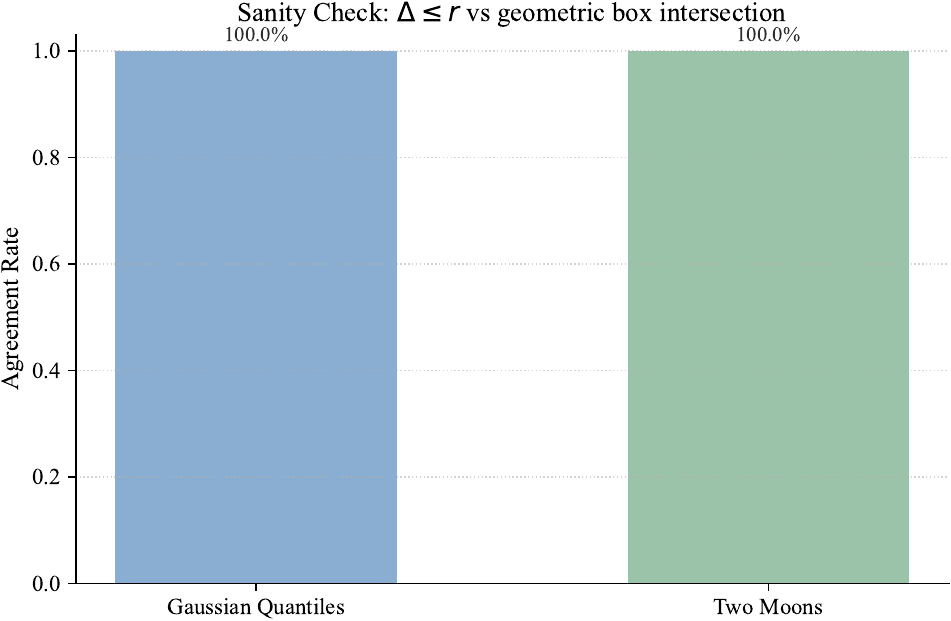}
  \caption{Exact $\ell_\infty$ window-cut criterion under fixed reference batches. The normalized-offset test matches explicit hyperplane--box intersection checks on both datasets.}
  \label{fig:linf_equivalence_sanity}
\end{figure}

\subsection{Offset Distributions During Training}
\label{subsec:mechanism_validation_proximity_cutrate}

Under fixed reference batches, we now examine the distributions of the normalized offsets \eqref{eq:normalized-offsets-unified}, which determine the window cuts and enter Hypothesis~\ref{hyp:H1}.

\paragraph{Cross-trial CDF diagnostics.}
We compare the normalized offsets of non-BN models with those of BN models under frozen training-time statistics, for $T=30$ reference batches and three checkpoints (epochs 0, 50 and 100). For a model $m\in\{\mathrm{BN},\mathrm{nonBN}\}$, a hidden layer $\ell$ and a trial, let
\begin{equation}
\label{eq:pooled-cdf}
F^m_\ell(r):=\frac{1}{T\,|\Lambda_\ell|}\sum_{t=1}^{T}\sum_{a\in\Lambda_\ell}\mathbf 1\{\Delta^m_{a,t}\le r\},
\qquad r\ge0,
\end{equation}
be the empirical CDF of the normalized offsets $\Delta^m_{a,t}$ of layer $\ell$, pooled over switching hyperplanes and reference batches $t$. We report the one-sided Kolmogorov--Smirnov statistic, the Wasserstein-1 distance and the CDF-area difference,
\begin{equation}
\label{eq:cdf-stats}
\begin{gathered}
D^+_\ell:=\sup_{r\ge0}\bigl(F^{\mathrm{BN}}_\ell(r)-F^{\mathrm{nonBN}}_\ell(r)\bigr),\\
W_{1,\ell}:=\int_0^\infty\bigl|F^{\mathrm{BN}}_\ell(r)-F^{\mathrm{nonBN}}_\ell(r)\bigr|\,dr,
\qquad
A_\ell:=\int_0^\infty\bigl(F^{\mathrm{BN}}_\ell(r)-F^{\mathrm{nonBN}}_\ell(r)\bigr)\,dr,
\end{gathered}
\end{equation}
computed per trial and summarized across trials. Since the offsets are nonnegative, $A_\ell$ equals the mean non-BN offset minus the mean BN offset, and $A_\ell=W_{1,\ell}$ exactly when $F^{\mathrm{BN}}_\ell\ge F^{\mathrm{nonBN}}_\ell$ everywhere. By Lemma~\ref{lem:linf-intersection}, $F^m_\ell(r)$ is the fraction of switching hyperplanes that cut the window of radius $r$ around the batch centroid (the event $\Delta=r$ has probability zero), and by Lemma~\ref{lem:exchangeability} it estimates the cut probability $p^m(r)$ averaged over reference batches. The ordering $F^{\mathrm{BN}}_\ell\ge F^{\mathrm{nonBN}}_\ell$ is therefore the empirical counterpart of Hypothesis~\ref{hyp:H1} at all radii simultaneously.

At epoch 100, $D^+_\ell$ lies between $0.97$ and $1$ in all three layers and on both datasets (Figure~\ref{fig:ks_dplus_and_area_traj}, left): at some radius, almost all BN hyperplanes and almost no non-BN hyperplanes cut the window. The CDF-area difference is positive at all three checkpoints and increases over training (Figure~\ref{fig:ks_dplus_and_area_traj}, right). Both statistics are consistent with Hypothesis~\ref{hyp:H1}. At epoch 0 all BN offsets vanish for ReLU (Remark~\ref{rem:init-central}), so the values at that checkpoint only reflect the non-BN offsets. The statistic $D^+_\ell$ alone does not show that $F^{\mathrm{BN}}_\ell\ge F^{\mathrm{nonBN}}_\ell$ at every radius; the full CDFs reported in Section~\ref{subsec:cdf_cross_dataset} for the real datasets show this ordering for every radius $r\ge0.01$.

\begin{figure}[!tbp]
  \centering
  \includegraphics[width=0.45\linewidth]{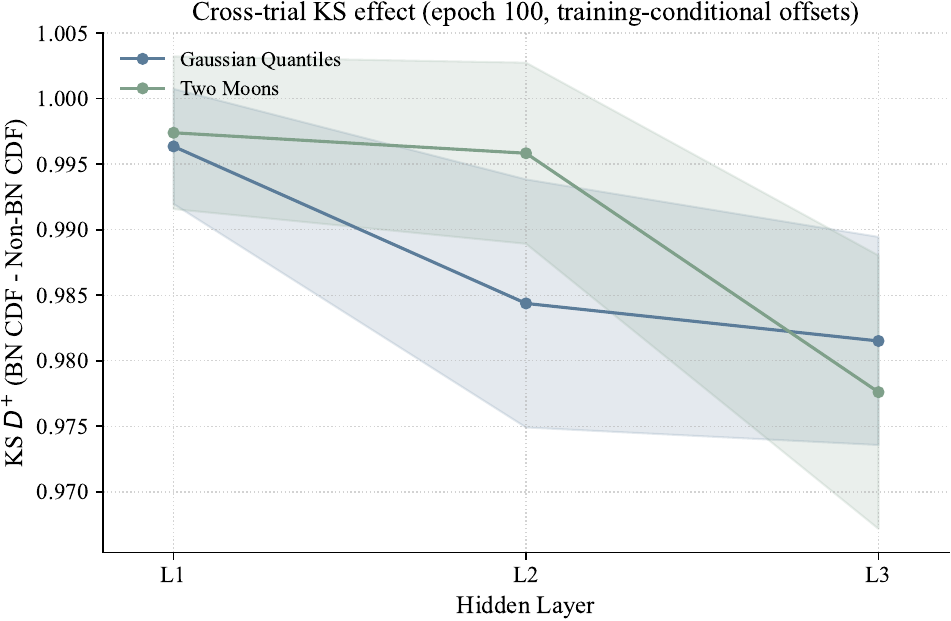}\hfill
  \includegraphics[width=0.45\linewidth]{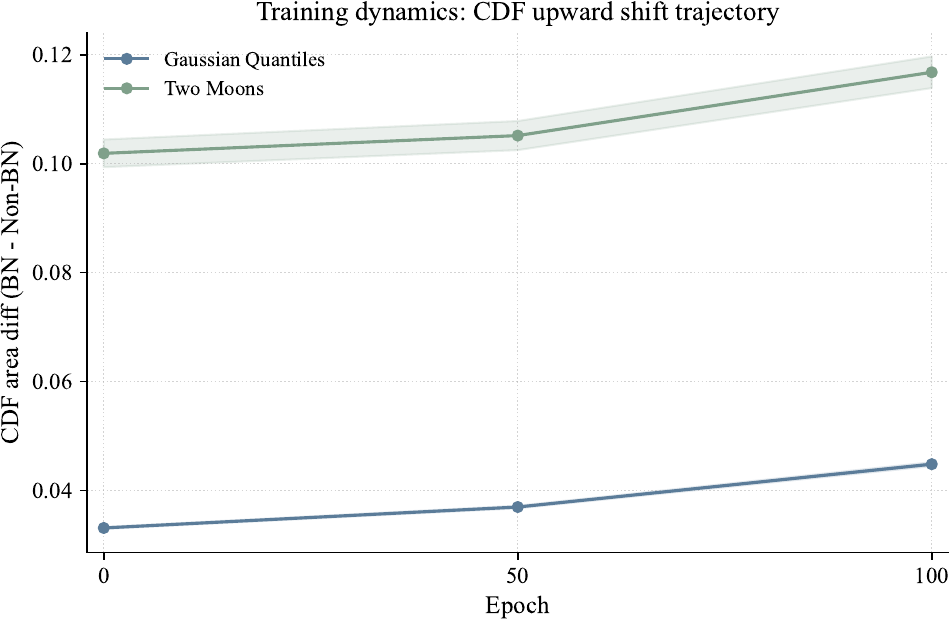}
  \caption{Training-conditional offset diagnostics across trials and checkpoints. \emph{Left:} one-sided KS statistics $D^+$ at epoch 100 are positive across trials. \emph{Right:} the CDF-area difference (BN minus non-BN) increases over training in these experiments.}
  \label{fig:ks_dplus_and_area_traj}
\end{figure}

\paragraph{Centroid-to-hyperplane Euclidean distance.}
We also report the Euclidean distance from the batch centroid to each switching hyperplane, $\mathrm{dist}_2(\bar u,H)=|c-\langle w,\bar u\rangle|/\|w\|_2$ for $H=\{u:\langle w,u\rangle=c\}$ \eqref{eq:l2-distance-hyperplane}. It is easier to interpret than the normalized offset but is not the exact $\ell_\infty$ cut criterion used in the theory. BN shifts these distance distributions towards zero across layers and datasets (Figure~\ref{fig:hyperplane_distance_layers1to5}).

\begin{figure*}[!tbp]
  \centering
  \includegraphics[width=0.95\linewidth,trim=0 0 654.14 28,clip]{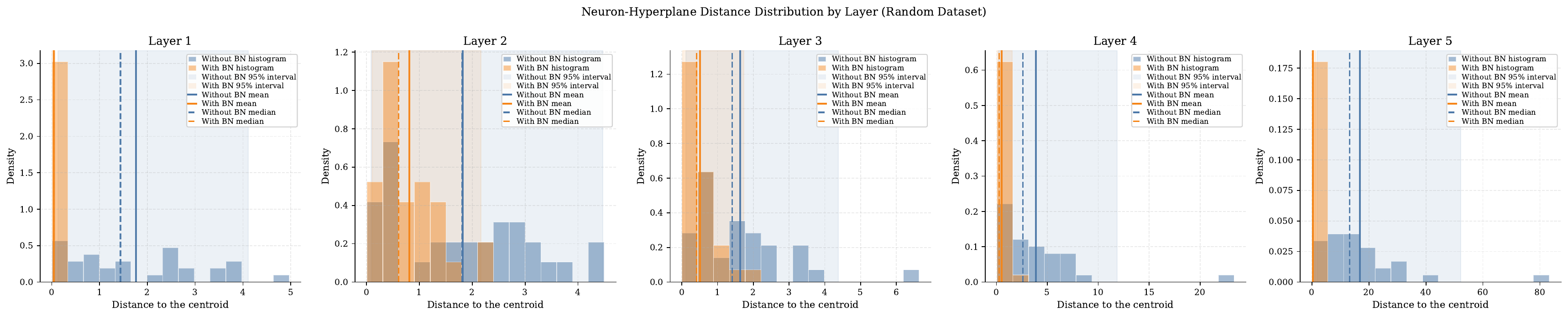}\\[1mm]
  \includegraphics[width=0.625\linewidth,trim=995 0 0 28,clip]{figs/hyperplane_distance_hist_layers1to5_random_icml.pdf}\\
  \small Random Uniform ($[32]^5$), hidden layers 1--3 (top) and 4--5 (bottom)\\[2mm]
  \includegraphics[width=0.95\linewidth,trim=0 0 0 28,clip]{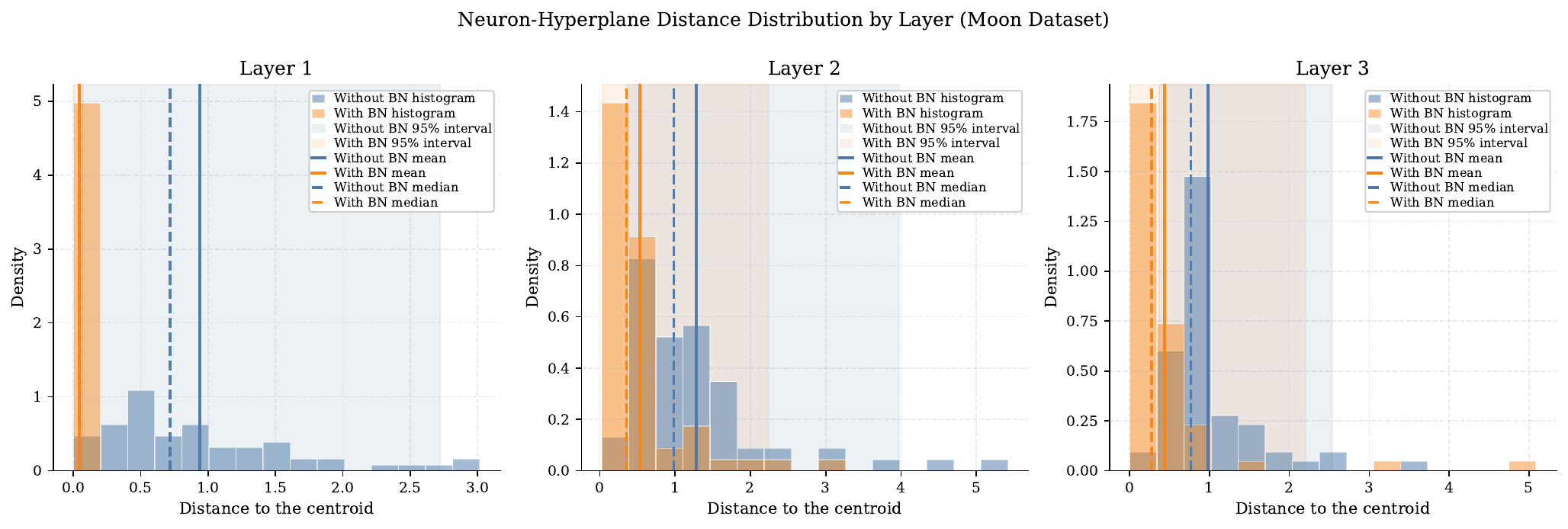}\\
  \small Two Moons, hidden layers 1--3\\[2mm]
  \includegraphics[width=0.95\linewidth,trim=0 0 0 28,clip]{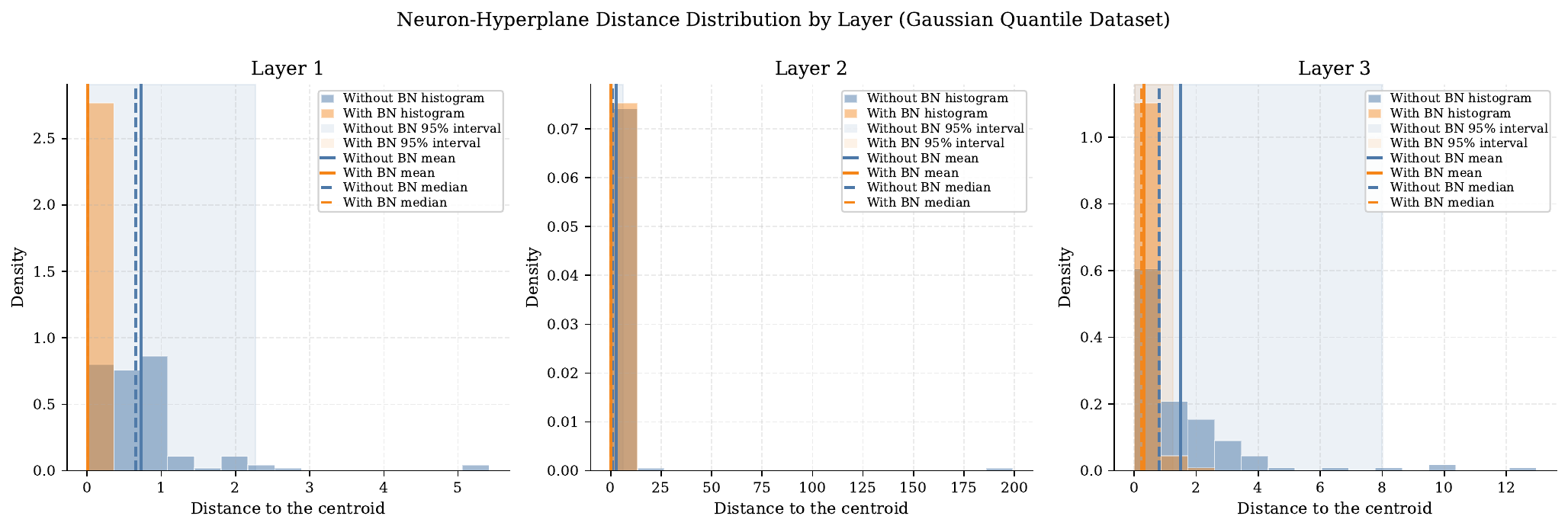}\\
  \small Gaussian Quantiles, hidden layers 1--3
  \caption{Centroid-to-hyperplane Euclidean distance distributions in representation space. BN shifts the distributions towards zero in every layer and dataset shown.}
  \label{fig:hyperplane_distance_layers1to5}
\end{figure*}

\paragraph{Inference-mode window-cut rates.}
We also report an inference-mode diagnostic based on running statistics. In this experiment, BN layers are evaluated in \texttt{eval()} mode, so BN uses checkpointed running statistics $(\bar\mu,\bar v)$ instead of instantaneous mini-batch statistics. We compute the inference-mode normalized offsets
\begin{equation}
\label{eq:offsets_run_for_cutrate}
\begin{aligned}
\Delta_{\ell,j}
&=
\frac{|w_{\ell,j}^{\top}\bar u_\ell+b_{\ell,j}|}{\|w_{\ell,j}\|_1},\\
\Delta^{\mathrm{BN,run}}_{\ell,j}
&=
\frac{|w_{\ell,j}^{\top}\bar u_\ell+b_{\ell,j}-\bar\mu_{\ell,j}
+\alpha_{\ell,j}\sqrt{\bar v_{\ell,j}+\varepsilon}|}{\|w_{\ell,j}\|_1},\\
\alpha_{\ell,j}
&:=\beta_{\ell,j}/\gamma_{\ell,j}.
\end{aligned}
\end{equation}
Using matched radii chosen by a fixed quantile rule, we find that BN models exhibit higher inference-mode window-cut rates across layers on the evaluated toy datasets (Figure~\ref{fig:cutrate_moon_gauss_main}). This diagnostic is not covered by the theory, which concerns the frozen-batch map; it shows that a similar pattern is visible when BN is evaluated with running statistics.

\begin{figure*}[!tbp]
  \centering
  \begin{minipage}[t]{\textwidth}
    \centering
    \includegraphics[width=0.95\linewidth]{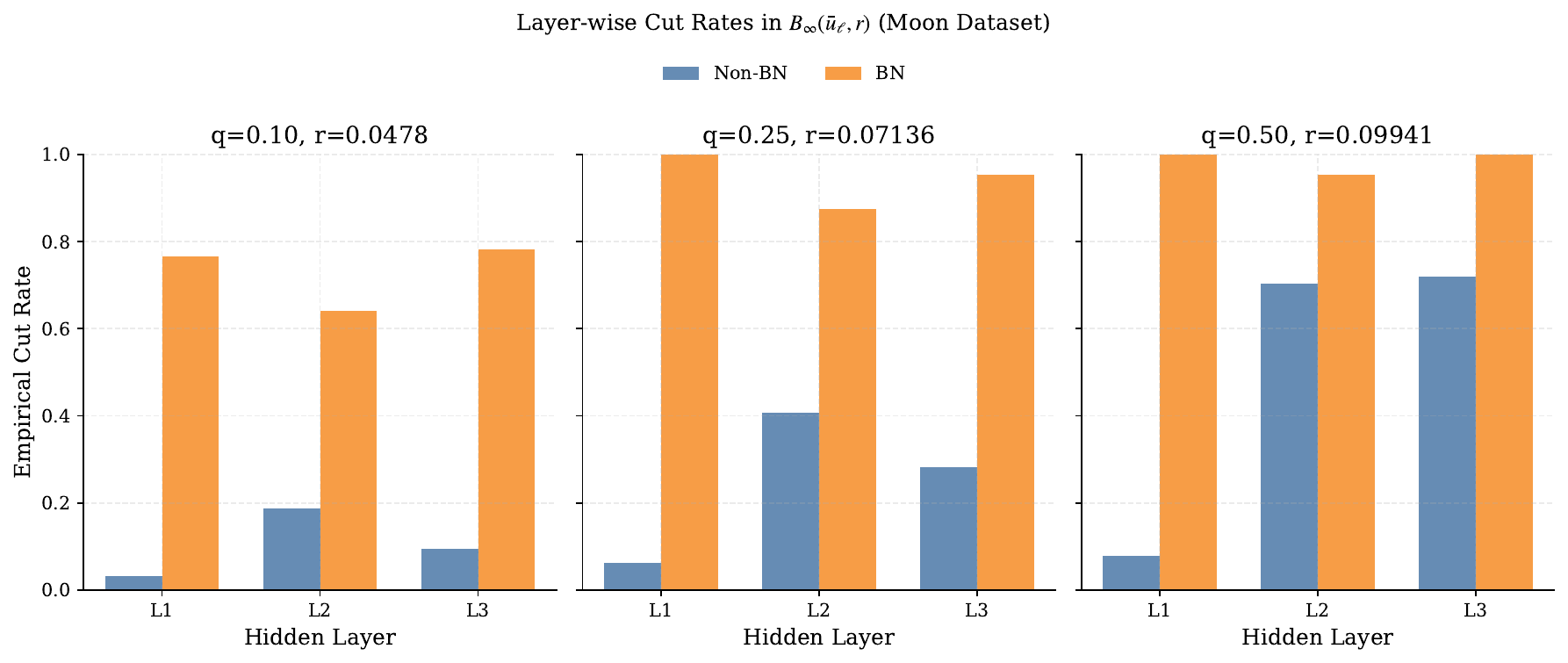}
    \vspace{-1mm}
    \small Two Moons: window-cut-rate bars ($q=\{0.10,0.25,0.50\}$)
  \end{minipage}\hfill
  \begin{minipage}[t]{\textwidth}
    \centering
    \includegraphics[width=0.95\linewidth]{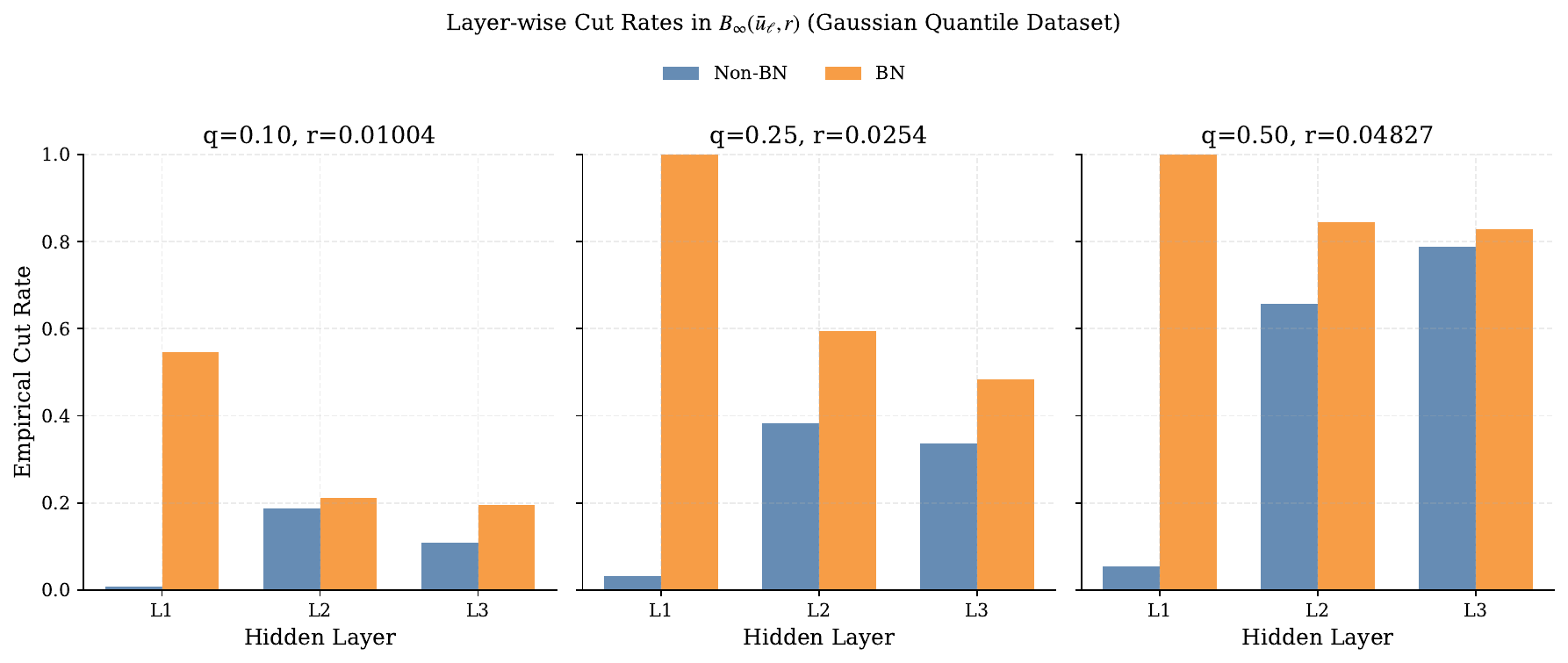}
    \vspace{-1mm}
    \small Gaussian Quantiles: window-cut-rate bars ($q=\{0.10,0.25,0.50\}$)
  \end{minipage}
  \caption{Inference-mode layerwise window-cut rates evaluated at radii selected by a fixed quantile rule. In these experiments, BN exhibits higher window-cut rates than non-BN across layers.}
  \label{fig:cutrate_moon_gauss_main}
\end{figure*}

\subsection{Multi-Layer Networks: Local Ingredients and Global Counts}
\label{subsec:exp2_multilayer}

For deep networks, we distinguish between local properties appearing in the theory and global region counts observed in the full input space.

\paragraph{Exact local region counts in deep networks.}
We first report exact input-space local region counts in centroid-anchored windows for deep MLPs. Training follows the protocol of Section~\ref{subsec:exp1}. The dataset--architecture pairs are:
\begin{enumerate}
    \item \textbf{Gaussian Quantiles:} $[128,128,128]$, with $\Omega=[-1,1]^2$.
    \item \textbf{Two Moons:} $[64,64,64]$, with $\Omega=[-1,2]^2$.
    \item \textbf{Random Uniform:} $[32,32,32,32,32]$, with $\Omega=[-0.5,4.5]^2$.
\end{enumerate}
Table~\ref{tab:exp2_region_counts} shows that BN yields larger exact local region counts in all deep configurations evaluated here. Figure~\ref{fig:exp2_deep_partition} shows representative input-space partitions of the BN and non-BN models at epoch 100.

\begin{table}[!tbp]
\caption{Exact local region counts in deep networks at epoch 100. Entries are reported as mean$\pm$std over 10 matched random seeds.}
\label{tab:exp2_region_counts}
\centering
\begin{tabular}{lllc}
\toprule
\textbf{Dataset} & \textbf{Architecture} & \textbf{Model} & $\mathbf{N_{\mathrm{reg}}(\Omega)}$ \\
\midrule
\multirow{2}{*}{Gaussian Quantiles}
& \multirow{2}{*}{3 Layers $[128\times3]$}
& non-BN & $29,351 \pm 16$ \\
& & BN & $\mathbf{161,100 \pm 22}$ \\
\midrule
\multirow{2}{*}{Two Moons}
& \multirow{2}{*}{3 Layers $[64\times3]$}
& non-BN & $5,868 \pm 8$ \\
& & BN & $\mathbf{30,434 \pm 9}$ \\
\midrule
\multirow{2}{*}{Random Uniform}
& \multirow{2}{*}{5 Layers $[32\times5]$}
& non-BN & $2,617 \pm 2$ \\
& & BN & $\mathbf{38,255 \pm 14}$ \\
\bottomrule
\end{tabular}
\end{table}

\begin{figure}[!tbp]
    \centering
    \includegraphics[width=0.95\linewidth]{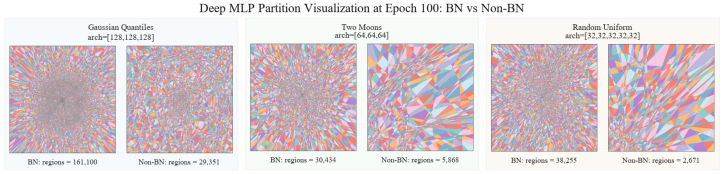}
    \caption{Representative input-space partitions for deep MLPs at epoch 100 across three dataset--architecture pairs. In each block, the BN partition is shown on the left and the non-BN partition on the right. The BN models exhibit finer visible partitions in the displayed neighborhoods.}
    \label{fig:exp2_deep_partition}
\end{figure}

\paragraph{Parent-region checks.}
We now check Proposition~\ref{prop:count-transfer} numerically inside fixed parent regions $R$ of the prefix map $h^{(\ell-1)}$.

\paragraph{Rank of the prefix map.}
For sampled points in the data neighborhood, we identify their parent regions for a depth-2 prefix map and extract the coefficient $\widetilde A_R$ (the Jacobian, which is constant inside $R$). We evaluate the drop-rank ratio $\mathbb P(\mathrm{rank}(\widetilde A_R)<2)$ and the smallest singular value $\sigma_{\min}(\widetilde A_R)$. The count transfer of Proposition~\ref{prop:count-transfer}(ii) holds for any rank, whereas the density statement (iii) requires rank $D_0=2$. Figure~\ref{fig:deep_lift_rank_sigma_stats} shows that the drop-rank ratio is zero on both datasets for both BN and non-BN models in the sampled neighborhoods, so that the density statement also applies there. BN models also have larger $\sigma_{\min}(\widetilde A_R)$ in these experiments.

\begin{figure}[!tbp]
  \centering
  \includegraphics[width=0.95\linewidth]{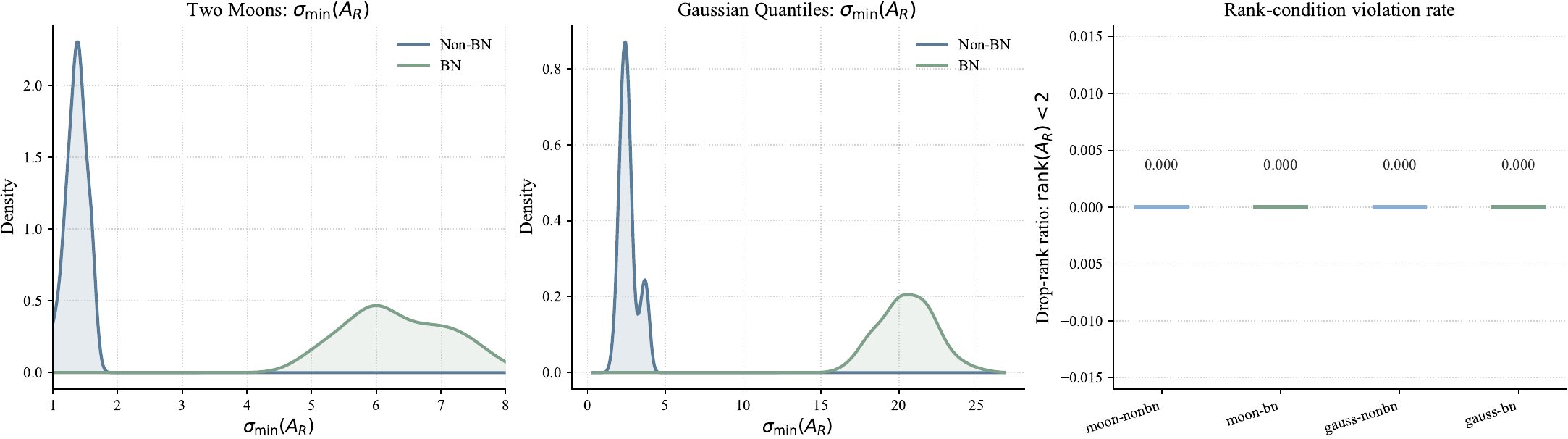}
  \caption{Rank of the prefix map inside sampled parent regions. The drop-rank ratio is zero in the sampled regions for both BN and non-BN models, and BN models have larger $\sigma_{\min}(\widetilde A_R)$ in these experiments.}
  \label{fig:deep_lift_rank_sigma_stats}
\end{figure}

\paragraph{Component-count equivalence inside a fixed parent region.}
We next check Proposition~\ref{prop:count-transfer}(ii). Fix a parent region $R$ sampled from the data neighborhood. Inside $R$ the prefix map $h^{(\ell-1)}$ is affine, so the proposition predicts that the local switching arrangement in input space is the pullback of the restricted arrangement in representation space, with the same number of connected components. To check this, we construct an intrinsic two-dimensional window that remains inside $R$ using a grid-based validity filter and retain only windows with sufficient in-region support. Across all retained windows, the connected-component counts agree exactly between input space and representation space, with Jaccard overlap 1.0. This agrees with Proposition~\ref{prop:count-transfer}(ii).

\subsection{Offset Distributions on Real Datasets}
\label{subsec:cdf_cross_dataset}

Exact affine-region counting in high dimensions is NP-hard under standard region definitions for ReLU networks~\citep{book47}. We therefore study the normalized offsets, which by Lemma~\ref{lem:linf-intersection} determine the window cuts exactly.

For CIFAR-10, MNIST, and TinyImageNet, we extract matched BN and non-BN checkpoints and compute, for each hidden layer $\ell\in\{1,2,3\}$, the empirical CDF of
\begin{equation}
\label{eq:exp3_normalized_offsets}
\Delta^{\mathrm{nonBN}}_{\ell,j}
=
\frac{|w_{\ell,j}^{\top}\bar u_{\ell}+b_{\ell,j}|}{\|w_{\ell,j}\|_1},
\qquad
\Delta^{\mathrm{BN}}_{\ell,j}
=
\frac{|\beta_{\ell,j}/\gamma_{\ell,j}|\,\sqrt{v_{\ell,j}+\varepsilon}}{\|w_{\ell,j}\|_1},
\end{equation}
where $\bar u_\ell$ is the empirical centroid of the layer-$\ell$ representations. These are the offsets \eqref{eq:normalized-offsets-unified} for ReLU, for which $\tau=0$ and $|\delta_{\ell,j}|=|\beta_{\ell,j}/\gamma_{\ell,j}|$. By Lemma~\ref{lem:linf-intersection}, $\Delta_{\ell,j}<r$ is exactly the event that the switching hyperplane of unit $j$ cuts the $\ell_\infty$ window of radius $r$ around $\bar u_\ell$. The empirical CDF $F^m_\ell(r)$ is therefore the fraction of switching hyperplanes that cut this window, that is, an estimate of $\mathbb E[f_1]/|\Lambda_\ell|$ in Theorem~\ref{thm:reduction} as a function of $r$, and comparing the two CDFs tests Hypothesis~\ref{hyp:H1} directly. For the first layer both models act on the same input space, so the comparison at a common radius is direct; for deeper layers each model has its own representation space, and the comparison at a common radius also reflects differences in the scale of the representations.

Figure~\ref{fig:cdf_three_datasets} shows that, for all three datasets and all three layers, $F^{\mathrm{BN}}_\ell(r)\ge F^{\mathrm{nonBN}}_\ell(r)$ for every radius $r\ge0.01$. For any such radius, at least as large a fraction of BN hyperplanes as of non-BN hyperplanes cuts the window, as Hypothesis~\ref{hyp:H1} requires. Below this radius the two curves differ by at most $0.05$, and only at radii where fewer than $8\%$ of the hyperplanes of either model cut the window; the largest differences occur in layer~3 of CIFAR-10 and of TinyImageNet, where a few non-BN hyperplanes have smaller offsets than every BN hyperplane.

\begin{figure*}[!tbp]
  \centering
  \includegraphics[width=\textwidth]{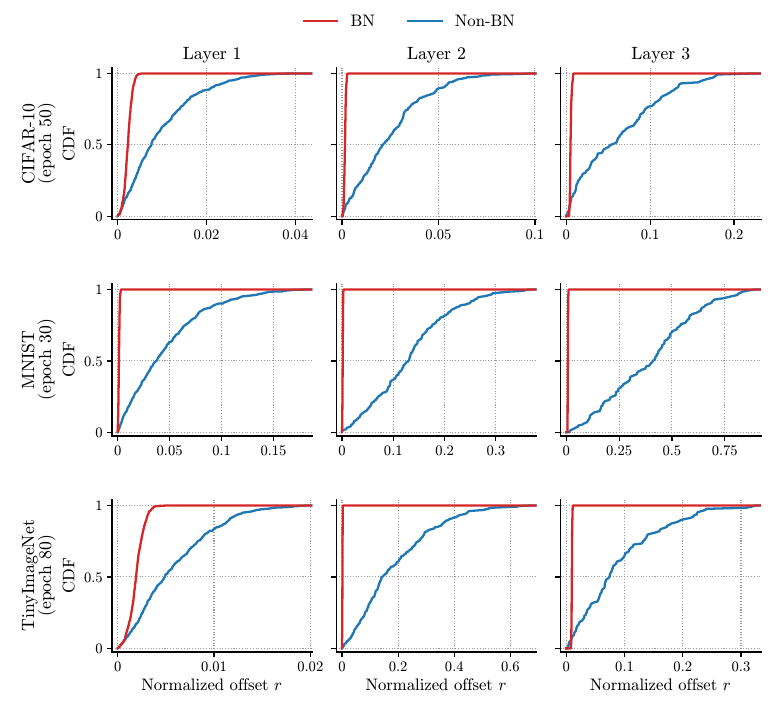}
  \caption{Empirical CDFs of the normalized offsets \eqref{eq:exp3_normalized_offsets} on three real datasets (rows) for hidden layers 1--3 (columns). In every panel the BN curve is at least as high as the non-BN curve for every radius $r\ge0.01$. Below this radius the two curves differ by at most $0.05$. Each curve is drawn over the whole horizontal range and equals $1$ beyond the largest offset of its model.}
  \label{fig:cdf_three_datasets}
\end{figure*}

\subsection{Affine-Region Structure on Two-Dimensional Slices}
\label{subsec:realdata_affine_regions_figure_only}

To obtain exact region counts on real-data models, we complement the offset distributions with matched two-dimensional slices.

For CIFAR-10, MNIST, and TinyImageNet, we fix matched two-dimensional projection planes in representation space. For each plane, we evaluate a trained BN model and its matched non-BN counterpart under identical architecture, optimization, and checkpoint selection, and compute the exact affine-region count on that slice. Across datasets and matched views, BN produces finer partitions and larger slice-wise affine-region counts (Figure~\ref{fig:realdata_affine_regions_bn_vs_nonbn}). These counts concern the chosen slices; they are not the local region counts of Section~\ref{subsec:exact-count}.

\begin{figure*}[!tbp]
  \centering
  \includegraphics[width=0.95\linewidth]{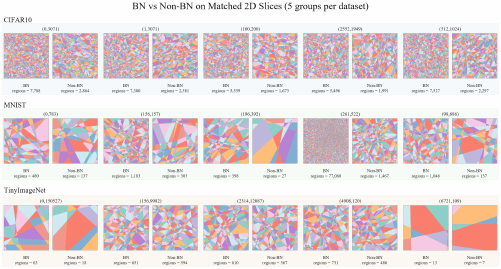}
  \caption{Affine-region partitions on matched two-dimensional slices for BN and non-BN models trained on CIFAR-10, MNIST, and TinyImageNet. Each BN/non-BN pair is evaluated on the same projection plane, and the exact affine-region count on that slice is reported on the corresponding panel.}
  \label{fig:realdata_affine_regions_bn_vs_nonbn}
\end{figure*}

\subsection{Decision-Boundary Formation}
\label{subsec:bn_boundary_dynamics_moon_gauss}

Finally, we examine whether differences in partition structure are reflected at the classifier level.

Let $f_t:\mathbb{R}^2\to\mathbb{R}^C$ denote the network at epoch $t$, with class logits
\begin{equation}
f_t(x)=\bigl(f_{t,1}(x),\dots,f_{t,C}(x)\bigr),
\end{equation}
and predicted label
\begin{equation}
\hat y_t(x):=\arg\max_{c\in\{1,\dots,C\}} f_{t,c}(x).
\end{equation}
The decision region of class $c$ at epoch $t$ is
\begin{equation}
\mathcal D_{t,c}:=\{x\in\mathbb{R}^2:\hat y_t(x)=c\},
\end{equation}
and the decision boundary is
\begin{equation}
\partial \mathcal D_t
:=
\bigcup_{c\neq c'}
\overline{\mathcal D_{t,c}}\cap \overline{\mathcal D_{t,c'}}.
\end{equation}
Because the networks considered here are CPA, each logit is affine on every affine region, so the decision boundary is assembled from region-wise affine pieces after composition through the hidden layers. The visualized decision boundary is therefore distinct from the hidden-layer switching hyperplanes themselves.

We consider Two Moons and Gaussian Quantiles, and train matched MLPs with architecture $[16,16,16]$ under two conditions: non-BN and BN, with all other hyperparameters fixed.

Figures~\ref{fig:bn_boundary_dynamics_moon_gauss} and \ref{fig:bn_accuracy_vs_epoch} show paired boundary visualizations and validation-accuracy trajectories. In these experiments, BN models exhibit earlier visually coherent decision boundaries and faster increases in validation accuracy.

\begin{figure*}[!tbp]
  \centering
  \includegraphics[width=0.95\linewidth]{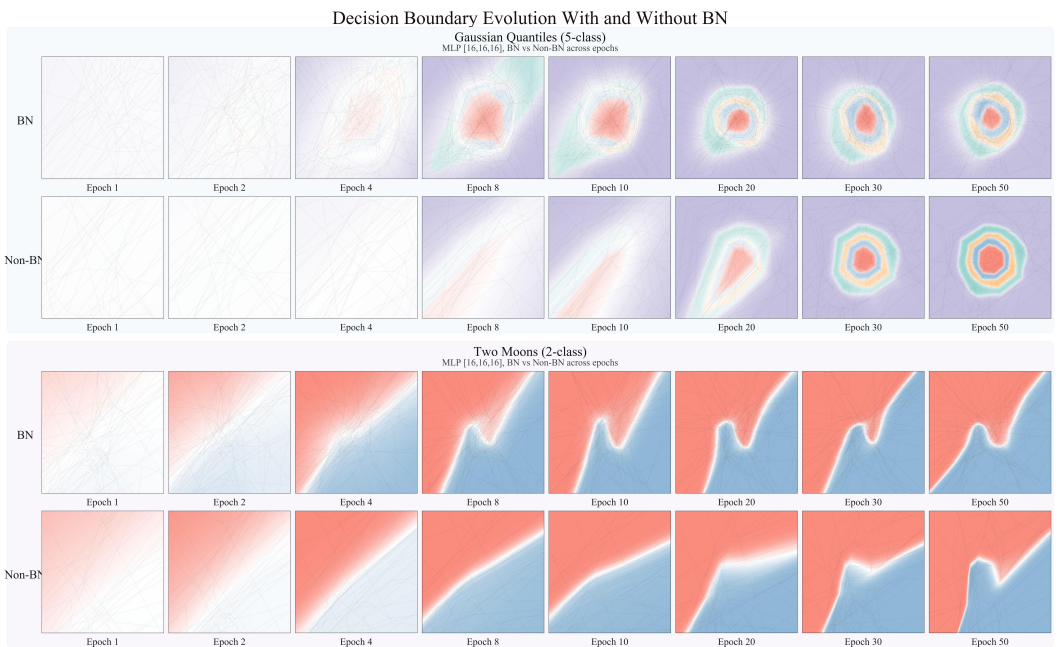}
  \caption{Decision-boundary evolution under matched BN and non-BN training on Two Moons and Gaussian Quantiles. Each colored region corresponds to the classifier prediction in input space, and the white interfaces provide a numerical approximation of the decision boundary.}
  \label{fig:bn_boundary_dynamics_moon_gauss}
\end{figure*}

\begin{figure}[!tbp]
  \centering
  \includegraphics[width=0.95\linewidth]{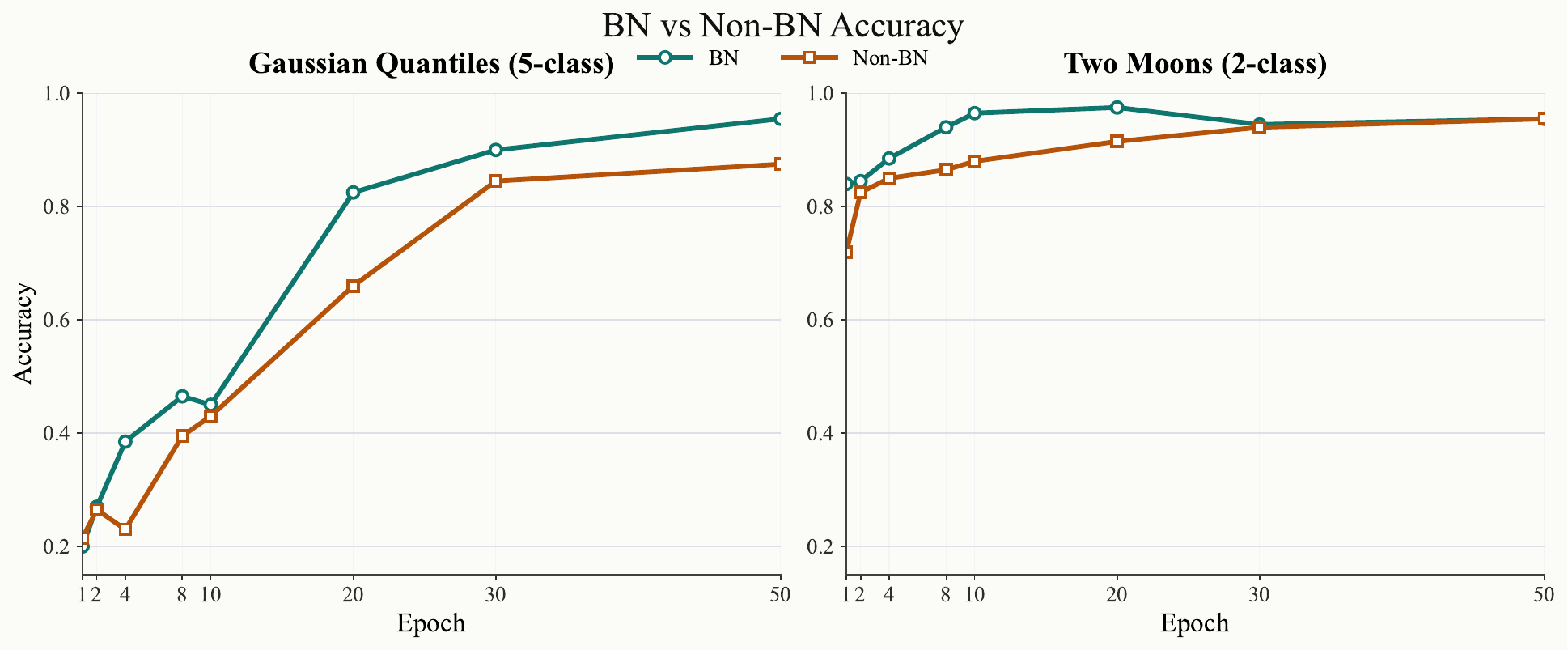}
  \caption{Validation accuracy over training epochs on Two Moons and Gaussian Quantiles under matched settings.}
  \label{fig:bn_accuracy_vs_epoch}
\end{figure}

\paragraph{Summary.}
In two dimensions, exact enumeration gives larger local region counts for BN models under matched conditions, and by Lemma~\ref{lem:exact-local-count} most of the difference comes from intersections inside the window. Under fixed reference batches, the normalized offsets of BN models lie closer to zero than those of non-BN models, on synthetic and real datasets. In the sampled parent regions the prefix map has full rank, and the component counts in input space and in representation space agree. The experiments thus support several of the assumptions made in Section~\ref{sec:bn-local-density-standard-linf}, namely Hypothesis~\ref{hyp:H1}, Hypothesis~\ref{hyp:H2} in two dimensions, and the full-rank condition of Proposition~\ref{prop:count-transfer}(iii), and they are consistent with a larger local region count under BN. They are observational and do not establish the hypotheses beyond the evaluated settings.

\section{Discussion}
\label{sec:discussion}

Inference-time BN is an affine reparameterization and does not change the class of functions a CPA network can represent. During training, however, the batch statistics place each switching hyperplane relative to the current mini-batch. Conditioned on a mini-batch, every BN switching hyperplane is a translate of a hyperplane through the batch centroid, at an offset given by the standardized quantity $\delta_a$ times the spread of the batch, and independent of the raw bias. With a learned $\beta$ the arrangement is in general not central, and the standardized offsets measure how far it is from the central configuration.

The analysis separates three kinds of statements. The switching geometry, the $\ell_\infty$ cut criterion and the exact count formula hold for every parameter value. The control of cuts and interior intersections by the standardized offsets is a deterministic consequence of these identities. The comparison of expected local region counts between BN and non-BN networks rests on Hypotheses~\ref{hyp:H1} and~\ref{hyp:H2}, which we state explicitly and test empirically.

The exact count formula also shows that window cuts alone do not determine local refinement. At initialization, a ReLU BN layer cuts every window around the batch centroid but produces few regions, because its arrangement is central. In trained models, most of the difference in local counts between BN and non-BN networks comes from intersections inside the window. Small standardized offsets favor both cuts and interior intersections, which is the mechanism described by Proposition~\ref{prop:bn-standardized-control}.

Inside parent affine regions, the switching geometry of a deeper layer is the pullback of a representation-space arrangement, and numbers of connected components are preserved for any rank of the prefix map. This connects the single-layer analysis to the partition of the input space, one parent region at a time. On higher-dimensional datasets, where exact counting is intractable, the normalized offsets still give direct access to the window-cut events, and their distributions behave as Hypothesis~\ref{hyp:H1} requires. The results indicate that BN, in addition to its role in optimization, affects where the switching structure of a CPA network is placed relative to the data during training.

\section{Limitations}
\label{sec:limitations}

\paragraph{The comparison rests on hypotheses.}
The exact identities, the count formula and the control of cuts and intersections by standardized offsets hold for every parameter value, but the comparison of expected local region counts between BN and non-BN networks rests on Hypotheses~\ref{hyp:H1} and~\ref{hyp:H2}. We test Hypothesis~\ref{hyp:H1} through offset distributions and Hypothesis~\ref{hyp:H2} through exact counts in two dimensions; we have no direct test of interior intersections of three or more hyperplanes. A proof would require a description of the training dynamics of the normalization parameters, which we leave open. The standard initialization ($\beta=0$) is a degenerate, central configuration outside the scope of the comparison (Remark~\ref{rem:init-central}).

\paragraph{Local and batch-conditional scope.}
The analysis concerns the batch-conditional CPA map obtained by freezing training-time statistics on a reference batch. It does not describe the full sequence of mini-batches encountered during optimization. The exact region counts of Section~\ref{sec:experiments} are computed with BN in inference mode; this map has the same form as the frozen-batch map, with running statistics in place of batch statistics, but it is not conditioned on a reference batch. The transfer through depth is formulated one parent region at a time; interactions across neighboring parent regions are not covered, and a comparison between models in deeper layers requires the hypotheses for the restricted windows of each model.

\paragraph{Limits of exact empirical verification.}
Exact verification is feasible only in low dimensions, where affine regions can be enumerated. In higher dimensions exact region counting is NP-hard in general~\citep{book47}, and the corresponding experiments rely on normalized-offset distributions and low-dimensional slices. These measure the window-cut events exactly but do not give the full local region count.

\section{Conclusion}
\label{sec:conclusion}

We studied training-time BN in CPA networks from the perspective of local partition geometry. Conditioned on a mini-batch, each BN switching hyperplane is a translate of a hyperplane through the batch centroid, at an offset measured in batch-standardized units and independent of the raw bias. Building on this structure, we derived an exact criterion for window cuts in $\ell_\infty$ windows and an exact formula for the local region count in terms of cuts and interior intersections. Under BN both are controlled by the standardized offsets and the spread of the batch, which gives a lower bound on the local count that holds uniformly over the biases and the location of the batch; without BN the best such bound is the trivial bound $1$. We also showed that the counts transfer through depth inside parent regions of any rank.

The comparison of expected local region counts between BN and non-BN networks was reduced to two hypotheses, on window cuts and on interior intersections. The experiments support both where they can be checked: offset distributions on synthetic and real datasets for the first, and exact region counts in two dimensions for the second, where most of the difference between BN and non-BN networks comes from intersections inside the window. Proving the hypotheses, developing scalable diagnostics for higher-order intersections in high dimensions, and comparing this mechanism with other normalization schemes are natural directions for future work.

\vskip 0.2in
\bibliography{references}

\end{document}